\documentclass{dcogpaper}

\usepackage{inconsolata}
\usepackage[inline]{enumitem}
\usepackage{array}
\usepackage{flafter}
\usepackage{cleveref}
\usepackage{tikz}
\usetikzlibrary{arrows.meta}
\graphicspath{{figures/}}
\newtheorem{proposition}{Proposition}
\newtheorem{remark}{Remark}
\AddToHook{env/proposition/begin}{\crefalias{section}{proposition}}
\AddToHook{env/remark/begin}{\crefalias{section}{remark}}
\crefname{appendix}{appendix}{appendices}
\Crefname{appendix}{Appendix}{Appendices}
\newcommand{\R}{\mathbb{R}}
\newcommand{\cand}{\mathcal{V}}
\newcommand{\experts}{\mathcal{E}}
\newcommand{\ie}{i.e.\ }

\title{A Jury, Not an Electorate:\\
       Label-Free Reliability Estimation\\
       for Ensembles of Language Models}
\author{Alessandro Rozza\thanks{Correspondence: [first name]@distributed.inc.}\and
Davide Mambelli\texorpdfstring{\footnotemark[1]}{}\and
Steven Motta\texorpdfstring{\footnotemark[1]}{}\and
Juri Taresch\texorpdfstring{\footnotemark[1]}{}}
\affiliation{distributed cognition}
\date{2026-10-14}
\shorttitle{A Jury, Not an Electorate}
\paperstatus{Preprint}
\paperlogo{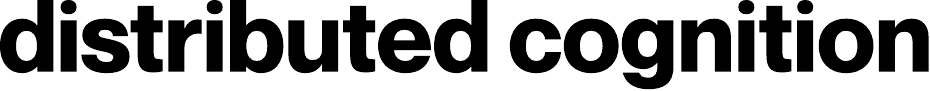}
\papersymbol{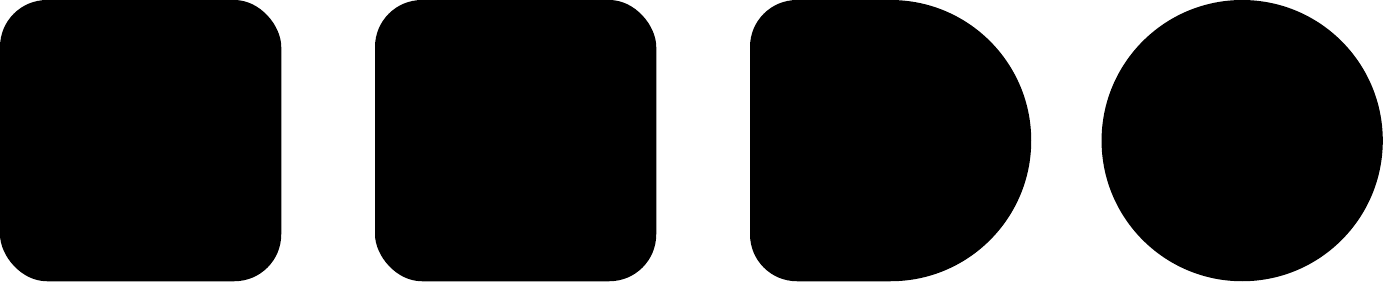}
\papercopyright{\textcopyright\ 2026 distributed cognition. All rights reserved.}

\begin{document}
\maketitle

\begin{abstract}
% An ensemble of language models is a jury, not an electorate. The answers are testimony from witnesses of unknown reliability, and the question is not what the majority said but who is credible. Most ensemble methods assign each
% expert a fixed reputation, often learned from labelled data. We ask whether that reputation can instead be
% inferred from the pattern of agreement and disagreement among the experts. %themselves.
% We model the experts as a jury: each question has a latent true answer, each expert reports it with an unknown
% reliability, and those reliabilities are estimated by Expectation Maximisation (EM) from unlabelled answer
% configurations alone. We show that its decision rule reduces exactly to a log-odds weighted linear score
% over the experts' answer distributions. What changes is not how the jury votes, but how it learns whom to trust.
% On a 1560-question benchmark spanning MMLU, GSM8K, ARC-Challenge, and HellaSwag,
% the jury reaches $76.9 \pm 1.3\%$ with nine experts of up to 3B parameters. It performs comparably to the same fusion rule
% with weights fitted from labels, so on this pool the jury learns effective expert weights
% from testimony alone, without measurable loss from removing weight supervision.
% It also beats the pool's best single expert by $2.0$ points ($95\%$ confidence interval $[+0.5,+3.6]$) and the
% strongest label-free alternative on the full benchmark by $5.5$ points.

An ensemble of language models is a jury, not an electorate. The answers are testimony from witnesses
of unknown reliability, and the question is not what the majority said but who is credible.
Most ensemble methods assign each expert a fixed reputation, often learned from labels.
We ask whether it can instead be inferred from the pattern of agreement and disagreement
among the experts.
We model the experts as a jury: each question has a latent true answer, each expert reports it with
unknown reliability, and those reliabilities are estimated by Expectation Maximisation (EM) from unlabelled answer configurations alone. We show that its decision rule reduces to a log-odds weighted linear score over the experts' answer distributions. What changes is not how the jury votes, but how it learns whom to trust.
On a 1560-question benchmark, the jury, with nine experts of up to 3B parameters, performs comparably to the same fusion rule with label-fitted weights, showing no statistically significant advantage from supervision.
Instead, it outperforms the best single model with matched self-consistency and the strongest label-free alternative on the full benchmark.

%It performs comparably to the same fusion rule with weights fitted from labels,
%while the jury learns effective expert weights from testimony alone, with no measurable loss
%from removing supervision.
%It also beats the pool's best model and the strongest label-free alternative on the full benchmark by $5.5$ pp.

%It also beats all pool’s models and the strongest label-free alternative by $5.5$ pp.

%The improvement comes from two different effects. Keeping each expert’s full answer distribution instead of only its argmax adds 8.4 points over majority voting. A sample-matched control shows that this gain comes from preserving the distribution itself, not simply from drawing more samples. Weighting the experts adds another 5.5 points, and the jury learns this weighting without labels.
\end{abstract}

% ===========================================================================
\section{Introduction}
% ===========================================================================

A collection of language models can know more than any one of its members. Models trained at different scales,
on different corpora, or for different domains fail in different places, so the answer that one model misses
may already be present elsewhere in the pool. The difficulty is not creating diversity, but deciding how to use it when models disagree.

Most existing approaches treat model aggregation like an electorate with known candidates.
Each model arrives with a reputation,
learned from labelled data or inferred from the query, and the final answer is chosen by applying those
reputations to the models' votes. In that view, credibility is decided before the answers are heard.

We take a different perspective: once all experts have answered, the question is no longer how much we already trust each model, but what their testimony tells us about them. The evidence is the configuration of answers: which answers appear, how strongly they are supported, and which experts agree. A jury can judge who to trust from the testimony itself.
%We take a different perspective: once all experts have answered, the question is no longer how much do we already trust each expert, but what does their testimony tell us about who to trust. The evidence is the configuration of the answers, the answers that appear, the degree to which they are supported, and the experts who agree with them.
%A jury can judge who to trust from the testimony itself. %It can infer credibility from the pattern of testimony itself.

A pool can improve on a plain vote in two separate ways.
The first comes from keeping the answer distributions: reducing each expert to its argmax discards the rest of
its distribution. The second comes from weighting the experts: even when the distributions are kept,
the experts still need to be weighted appropriately. %The two effects are distinct.
The two are often mixed together, but they are conceptually different.
On our benchmark, keeping the answer distributions improves majority voting by 8.4 points,
and weighting the experts adds a further 5.5 points. A sample-matched control shows that the first gain
comes from the distribution itself rather than from additional test-time sampling.

Our contribution concerns the second gain: learning whom to trust without labels.
%We isolate the effect of retaining the answer distributions from both expert weighting and additional sampling compute.
We separate the effect of retaining the answer distributions from the effects of expert weighting and additional sampling.
We use a simple generative model, inspired by \citet{dawid1979em}, where each question has a latent true answer and each expert has an unknown reliability.
%We posit a simple generative model in which each question has a latent true answer and each expert reports it with an unknown reliability.
Those reliabilities are estimated by EM from unlabelled answer configurations alone.
The model never sees a gold answer. It sees the same training questions as the supervised baselines,
but only the experts' testimony.

Mathematically, the resulting decision rule is not a new family of ensemble methods. With one answer per expert,
it reduces exactly to the classical log-odds weighted vote. With answer distributions, it becomes a weighted
linear score over those distributions, equivalent to a linear opinion pool when the weights are non-negative
and a signed linear fusion otherwise.
Neither the vote nor the estimator is new in form: the first is a standard opinion pool, the second is an EM for a classical crowdsourcing model. What this work adds is the adaptation of that model to language-model outputs, and a direct test of whether reliabilities inferred from agreement can replace reliabilities measured on labels.

This allows a direct comparison with the same fusion rule whose weights are fitted from labelled questions,
which we call the supervised linear fusion. On a 1560-question benchmark spanning six MMLU subjects, GSM8K,
ARC-Challenge, and HellaSwag, nine experts of up to 3B parameters combined by the jury reach $76.9 \pm 1.3\%$.
This is 2.0 points above the pool's best single expert, with a 95\% confidence interval of $[+0.5,+3.6]$,
1.6 points above that expert when it is given the same eight-sample self-consistency on GSM8K ($[+0.0,+3.1]$), and
substantially above the best label-free alternative on the full benchmark. The jury also performs comparably to
the supervised linear fusion.

%On this pool, supervised fitting provides no measurable advantage over estimating the weights from agreement alone.

Our contributions are fourfold.
\begin{enumerate*}[label=(\arabic*)]

\item We separate two gains that are often mixed together. The first is the use of answer distributions instead of argmax answers, which applies to any method. The second is how the experts are weighted.
%\item We introduce an EM-based label-free estimator of expert reliability based on a simple generative model of the answer configuration. % we don't introduce it
\item We adapt the one-coin Dawid–Skene model to ensembles of language models with observed candidate sets and sampled answer distributions, and use it to obtain expert weights for linear fusion without labels.
%\item We prove label-free identifiability under task-defined support and characterise the resulting rule exactly. On single answers it reduces to a log-odds weighted vote, and on answer distributions to a signed linear fusion. The contribution is therefore the label-free estimation of the weights, not the fusion rule itself.
\item We characterise the resulting rule exactly: on single answers it is the log-odds weighted vote, and on answer distributions a signed linear fusion. This makes the same rule with label-derived reliabilities the natural supervised reference. We also give a closed-form, moment-based identifiability result for the one-coin model with $K \ge 3$ that does not require any expert to be better than chance.
\item We extensively evaluate the method by ablating its assumptions and comparing it against ten post-hoc combiners under a common full-benchmark protocol, spanning both supervised and label-free methods, including SkillAggregation, with a $+2.7$ pp margin over its strongest tested variant.

%Our contributions are fourfold.
%\begin{enumerate*}[label=(\arabic*)]
    %\item We separate two gains that are often mixed together. The first being %the usage of the answer distributions over the argmax of all answers, which %extends to any method. The second is how the experts are weighted.

    %\item We introduce an EM-based label-free estimator of expert reliability %based on a simple generative model of the
    %answer configuration.

    %\item We characterise the resulting rule exactly. On single answers it %reduces to a log-odds weighted vote,
    %and on answer distributions to a signed linear fusion. This identifies weight %estimation as the methodological
    %contribution.

    %\item We extensively evaluate the method both ablating its assumptions and %compaing it against ten post-hoc combiners under a common full-benchmark %protocol, spanning both supervised and label-free methods, such as %SkillAggregation which we outperform by 2.7 pp.

    %including supervised versions of the same functional form, paired and nested bootstrap analyses,
    %sample-matched controls, task-specific reliabilities, and ablations of the model assumptions.
    %We additionally compare with SkillAggregation on the multiple-choice subset, and replicate the comparison on every sub-pool of three, five and seven of the nine experts.
    %The jury is the strongest label-free method we test; it is +0.6 points above the supervised linear fusion given the fitted models, with no statistically significant difference on fresh questions.
\end{enumerate*}

% ===========================================================================
\section{Related work}
\label{sec:related}
% ===========================================================================

We follow the taxonomy of \citet{chen2025llmensemble}, which divides LLM
ensembling into what happens before, during and after inference.

\paragraph{Before inference: routing.}
% Routers select one model before generation using correctness labels
% \cite{ding2024hybrid,feng2024graphrouter}, reward-model scores
% \cite{lu2023zooter}, human preferences \cite{ong2024routellm}, bandit
% objectives \cite{metallm2024,pickllm2024}, or information-theoretic
% criteria under a budget \cite{infotheoretic2026}.
A router maps the query to one model, trained on
correctness labels \cite{ding2024hybrid,feng2024graphrouter}, on reward-model
scores \cite{lu2023zooter}, or on human preferences \cite{ong2024routellm};
others cast the choice as a bandit \cite{metallm2024,pickllm2024} or select a
subset by mutual information under a budget \cite{infotheoretic2026}.
Simple embedding-based routers can also be competitive
\cite{knnrouter2025,avengers2025}.
This family is complementary to ours: it decides whom to ask, while we
decide whom to believe once everyone has answered.

%\paragraph{Before inference: routing.}
%This is the largest category. A router maps the query to one model, trained on
%correctness labels \cite{ding2024hybrid,feng2024graphrouter}, on reward-model
%scores \cite{lu2023zooter}, or on human preferences \cite{ong2024routellm};
%others cast the choice as a bandit \cite{metallm2024,pickllm2024} or select a
%subset by mutual information under a budget \cite{infotheoretic2026}.
%Recent work shows that simple routers can be competitive with more elaborate ones:
%a simple $k$-NN over embeddings performs comparably to learned routers \cite{knnrouter2025},
%and a similarity-based method matches more complex routing methods \cite{avengers2025}.
%This family is complementary to ours: it decides whom to ask, while we decide whom to believe once everyone has answered.

\paragraph{During inference: token, span and process fusion.}
These methods combine models during generation through vocabulary alignment
\cite{huang2024deepen,yu2024gac}, perplexity-based weighting
\cite{mavromatis2024packllm}, span-level cross-scoring
\cite{coolfusion2024}, speculative decoding \cite{specensemble2025},
or process-reward search over partial reasoning chains \cite{lemcts2024}.
They rely on logits or partial generations, while our method only uses final
answers.

%\paragraph{During inference: token, span and process fusion.}
%Methods belonging to this family interleave the models' generations. The hard technical problem is
%that heterogeneous models do not share a vocabulary, which DeePEn
%\cite{huang2024deepen} solves with a relative representation and GaC
%\cite{yu2024gac} with a union vocabulary; PackLLM \cite{mavromatis2024packllm}
%weights models at test time by perplexity. Span-level variants generate and cross-score fragments
%\cite{coolfusion2024} or borrow speculative decoding to make an ensemble cheaper
%than the sum of its parts \cite{specensemble2025}, and process-level variants
%search over partial reasoning chains with a process reward model
%\cite{lemcts2024}. These methods rely on logits or partial generations, while our method only uses final answers.

% \paragraph{After inference: the family this paper belongs to.}
% Post-hoc aggregation either weights experts or scores completed answers.
% Expert-weighting methods include majority voting \cite{li2024moreagents},
% unsupervised quality estimation \cite{guha2024smoothie},
% diversity-based selection \cite{tekin2024topla}, and the correlation-aware
% Optimal Weight and Inverse Surprising Popularity (ISP) rules of
% \citet{beyondmajority2025}.
% Answer-scoring methods instead use pairwise cross-encoders
% \cite{jiang2023llmblender}, supervised agreement and confidence features
% \cite{si2023more}, or hidden-state confidence estimates \cite{lens2025}.

\paragraph{After inference: the family this paper belongs to.}
Aggregating completed answers splits into two groups. \emph{Weighting the experts} covers
plain majority voting over sampled outputs \cite{li2024moreagents}, unsupervised
per-sample quality estimation \cite{guha2024smoothie}, diversity-driven subset
selection with a learned combiner \cite{tekin2024topla}, and the correlation-aware
Optimal Weight and Inverse Surprising Popularity (ISP) rules of
\citet{beyondmajority2025}.
%\emph{Scoring the answers} is the smaller group: 
\emph{Answer-scoring} methods form the smaller group: \citet{jiang2023llmblender} rank candidates with a pairwise cross-encoder over their text and then generate a fused
answer, \citet{si2023more} select among candidates with a supervised model over
inter-expert agreement and confidence, and LENS \cite{lens2025} learns per-model
confidence from hidden states.

Our method is a weighted linear fusion of the experts' answer distributions, with
the reliability vector as its only learned content (\Cref{sec:method}). We estimate those
reliabilities directly from the experts' unlabelled answer configurations. Other
label-free methods that weight the experts estimate a \emph{model's} quality: Smoothie \cite{guha2024smoothie} from
pairwise distances between outputs, and ISP \cite{beyondmajority2025} from
second-order agreement statistics that detect answers more popular than predicted.
\citet{beyondmajority2025} also propose OW-L and OW-I, which estimate each
model's accuracy without labels, from the pairwise agreement table or from ISP's
verdicts used as pseudo-labels, and then apply Optimal Weight; these are the
closest published counterparts of our estimator. %, and \Cref{app:skillagg} compares against them.

A close reference-free setting is SkillAggregation
\cite{sun-etal-2025-skillaggregation}, which learns to combine multiple LLM
judgments without ground-truth labels and accepts both binary and probabilistic
estimates. It extends Crowdlayer \cite{rodrigues2018deep} with a learned context encoder and per-judge skill
estimates, including context-dependent skills. Our method fits a restricted generative model to unlabelled
answer configurations, yielding log-odds weights for a fixed linear fusion rule. %(\Cref{prop:pool}).

\paragraph{Opinion pooling.}
Combining probability distributions is a classical problem in statistics.
With non-negative weights, our rule is Stone's linear pool
\cite{stone1961opinion}; negative log-odds weights instead yield a signed
linear fusion. The main alternative is the logarithmic pool, a weighted
geometric mean related to the product of experts \cite{hinton2002products};
the principles distinguishing the two are reviewed by
\citet{genest1986combining}. DeGroot's consensus model instead iteratively
averages expert opinions \cite{degroot1974consensus}.

In this literature, weights are obtained from experts or fitted on
questions with known answers, and equal weighting is often competitive
\cite{clemen1999combining}. In our work, the generative model of
\Cref{sec:method} determines the linear-fusion weights, and EM estimates them
from agreement alone.

%\paragraph{Opinion pooling.}
%Combining several probability distributions over one outcome space into a single
%distribution is an old problem in statistics, and its oldest family is the
%\emph{linear} pool of \citet{stone1961opinion}.
%With non-negative weights, our rule is a linear pool; because log-odds weights can be negative,
%in general it is a signed linear fusion.
%The standing alternative is the \emph{logarithmic}
%pool, a weighted geometric mean, which reappears in machine learning as the
%product of experts \cite{hinton2002products}. The two differ in what they preserve: a linear pool keeps a
%candidate alive when one expert insists on it, while a logarithmic pool kills it when one expert rules it out.
%The axioms that select between them are summarised by \citet{genest1986combining}.
%\citet{degroot1974consensus} gives the other classical construction, in which the
%experts converge on a consensus by repeatedly averaging each other rather than
%being averaged by an outside party.

%In this literature, weights are usually obtained from expert judgments or fitted
%on questions with known answers, and a recurring finding is that equal weights are
%hard to beat \cite{clemen1999combining}. Here, the generative model of
%\Cref{sec:method} determines the linear-fusion weights and EM estimates them from
%agreement alone.

\paragraph{Cascades.}
Cascades escalate through models until an answer is accepted
\cite{chen2023frugalgpt}, using stopping rules
\cite{cascaderouting2024}, learned policies \cite{der2024}, or
self-consistency \cite{modelswitch2025}. They address the accuracy-cost
trade-off rather than post-hoc aggregation.

%\paragraph{Cascades.}
%A third line escalates through models until an answer looks good enough
%\cite{chen2023frugalgpt}, with optimal stopping derived in closed form
%\cite{cascaderouting2024}, learned as a Markov decision process \cite{der2024}, or driven by
%self-consistency \cite{modelswitch2025}. Cascades address the accuracy–cost trade-off in a different way.

\paragraph{Classical roots.}
Recovering a latent truth and the
annotators' reliabilities jointly, from agreement and without labels, is
\citet{dawid1979em}; weighting voters by the log-odds of their accuracy is optimal
for conditionally independent voters \cite{nitzan1982optimal}, with the $K$-ary
form we use given in \citet{beyondmajority2025}; conditioning a combination on the
input is the
mixture-of-experts gate \cite{jacobs1991moe,shazeer2017moe};
%and detecting answers more popular than predicted is the approach proposed in \cite{prelec2017surprising}.
and \citet{prelec2017surprising} propose the approach for detecting answers that are more popular than predicted.
%Our reading of the LLM-ensembling literature is that
%it reached for supervision before exhausting what this label-free machinery
%already offers, and the experiments below quantify that within one nine-model pool and across its sub-pools.

% ===========================================================================
\section{Method}
\label{sec:method}
% ===========================================================================

\Cref{sec:why} shows which answers a weighted vote over the experts' argmax answers can reach, and identifies two separable
sources of improvement. \Cref{sec:model} introduces the generative model and
derives the EM updates that estimate each expert's reliability without labels.
\Cref{sec:rule} characterises the resulting decision rule, and \Cref{sec:ablations} outlines the modelling choices. %evaluated in \Cref{sec:ablresults}.

\subsection{What a weighted vote over argmax answers can and cannot do}
\label{sec:why}

Fix a question $q$ and let $\experts$ be the finite set of experts in the pool.
Each expert $e \in \experts$ returns an answer $a_e$ from a task-dependent answer
space $\mathcal{A}$; only the finite set of distinct answers actually produced, the candidate set
$\cand(q) = \{a_e : e \in \experts\} \subseteq \mathcal{A}$, enters the model.
%with $1 \le |\cand(q)| \le |\experts|$.
A weighted vote with weights
$w \in \R^{|\experts|}$ returns
\begin{equation}
  \hat{y}(w) \;=\; \operatorname*{arg\,max}_{v \in \cand(q)} \; \sum_{e\,:\,a_e = v} w_e .
  \label{eq:vote}
\end{equation}
We let $w$ range over $\R^{|\experts|}$ rather than over the simplex because the
model-derived weights can be negative for sufficiently weak experts, and the vote
and the model should speak about the same object.

A weighted vote over argmax answers can only return an answer in $\cand(q)$.
%A weighted vote over argmax answers can only select from in $\cand(q)$.
Conversely, any $v \in \cand(q)$ can be made the winner by assigning all weight
to an expert that proposed it. Changing the weights can therefore change which
observed answer wins, but cannot enlarge the set of possible winners. Two
separable gains lead beyond such a rule: keeping the answer distributions, which
can enlarge the reachable set, and weighting the experts better.

\subsection{A working generative model of the answer configuration}
\label{sec:model}

Rather than score a candidate from features of its supporters, we model how the
answers arose. For a question with candidate set $\cand$, let $z \in \cand$ be the
latent true answer. Each expert $e$ reports it with reliability $\pi_e$, and
otherwise reports one of the remaining candidates uniformly:
\begin{equation}
  P(a_e = v \mid z) \;=\;
  \begin{cases}
    \pi_e, & v = z,\\[2pt]
    \dfrac{1 - \pi_e}{m}, & v \neq z,
  \end{cases}
  \label{eq:emission}
\end{equation}
with $m = |\cand| - 1$. When $m = 0$ the pool is unanimous, $\cand$ is a singleton, \eqref{eq:emission} is
vacuous and the verdict is that answer; such questions are also excluded from the
estimation below, since they constrain no parameter.

\paragraph{What $\cand$ is, and what the model is defined on.}
$\cand(q)$ is not a fixed label space. It is the set of distinct answers produced
by the pool on $q$: in the greedy regime the distinct argmax answers, and with
repeated samples the union of all answers produced by the experts, over which
$p_e$ is the expert's empirical frequency. If repeated samples are unavailable,
$p_e$ reduces to a point mass at the greedy answer; a non-answer is omitted from
the likelihood. During fitting and prediction the set is treated as a fixed
support, and \eqref{eq:emission} is a simplified emission model defined on that
support: it does not model how the support itself was generated. This matters
for what the model can claim. Conditioning on the observed support would itself
induce dependence between the experts (if two experts can say $A$ or $B$ and the
observed set is $\{A, B\}$, they are known to disagree), so the conditional
independence used below is an assumption of the simplified model, not a
consequence of the construction. The error channel is uniform over the
alternatives in the support rather than over the task's full answer space, and a
candidate outside $\cand(q)$ has probability zero. The reliabilities $\pi_e$ are
therefore reliabilities \emph{according to this model}, and need not equal the
experts' true probabilities of being correct. \Cref{app:support} reports a
control in which the support of every multiple-choice question is fixed to its
four options instead of being read off the answers.

The model assumes that, given $z$, the experts are
conditionally independent, so with a uniform prior over $\cand$ the verdict is
\begin{equation}
  \hat{y} \;=\; \operatorname*{arg\,max}_{v \in \cand} \; \sum_{e \in \experts} \log P(a_e \mid z = v).
  \label{eq:verdict}
\end{equation}

This is a restricted form of \citet{dawid1979em}: one reliability per expert
rather than a full confusion matrix, and a uniform error distribution over the
remaining candidates. We use one reliability per expert, avoiding the
$O(|\experts|\,|\cand|^{2})$ parameters of a full confusion-matrix model.

\paragraph{Estimating $\pi$ without labels.}
The reliabilities are estimated by EM. For a distribution $p_{e,q}$
over $\cand(q)$ under the likelihood of \Cref{prop:pool}, the M-step is
\begin{equation}
  \pi_e \;=\; \frac{1}{|Q_e|}\sum_{q \in Q_e}\;\sum_{v \in \cand(q)}
             \gamma_q(v)\, p_{e,q}(v),
  \label{eq:mstep-dist}
\end{equation}
where $\gamma_q(v)$ is the posterior of the E-step and $Q_e$ are the
questions $e$ answered. An expert's reliability is the posterior-weighted
mass it placed on each candidate, averaged over questions.
For point-mass evidence $p_{e,q} = \delta_{a_e}$, this reduces to
$\pi_e = |Q_e|^{-1}\sum_{q \in Q_e}\gamma_q(a_e)$.

Writing $A_e$ for the double sum in \eqref{eq:mstep-dist}, the expected complete-data log-likelihood of expert $e$ is
$A_e \log \pi_e + (|Q_e| - A_e) \log(1 - \pi_e)$; maximising gives $\pi_e = A_e / |Q_e|$.
%The derivation is one line: writing $A_e$ for the double sum in \eqref{eq:mstep-dist}, the expected complete-data log-likelihood of expert $e$ is $A_e \log \pi_e + (|Q_e| - A_e) \log(1 - \pi_e)$ up to terms free of $\pi_e$ --- the per-question $\log m_q$, which does not survive the derivative even though $m_q$ varies with the question --- because $\sum_v \gamma_q(v) = 1$ absorbs the off-diagonal mass; maximising gives $\pi_e = A_e / |Q_e|$.
\paragraph{Identifiability and optimisation.}
The EM updates explain how to estimate the reliabilities without labels. A separate question is whether the distribution of unlabelled responses determines them uniquely. Under task-specified support, as in \Cref{app:support}, the model yields the following label-free identifiability guarantee.
\begin{proposition}[Label-free identifiability]
\label[proposition]{prop:identifiability}
Fix a support size $K \ge 3$ that occurs with positive probability,
and condition on questions with $|V(q)|=K$.
Let $p_e$ be the empirical distribution of expert $e$'s answers
sampled according to (2) on that support.
Assume $(p_e)_{e\in\mathcal E}$ are jointly observed and
conditionally independent given $(z,q)$, with question-independent
reliabilities.
If at least three experts satisfy $\pi_e \neq 1/K$, then the
reliability vector $\pi$ is identifiable from the joint distribution
of the $(p_e)_{e\in\mathcal E}$ alone.
\end{proposition}
The proof in \Cref{app:proofident} reconstructs each reliability from second and third order agreement moments. This guarantee holds even when every expert's probability of answering correctly is below  $1/K$, the probability of uniform random guessing. Even when the reliabilities are identifiable, EM may converge to a local rather than a global maximum of the non-concave likelihood.

We initialise $\pi_e = 0.6$ for every expert, giving all experts the same positive initial weight. We run a single EM fit with
no restarts, stop when no parameter moves by more than $10^{-7}$ or after $200$
iterations, and clip $\pi_e$ to $[0.02, 0.98]$ to keep the log-odds finite.
A $25$-start sensitivity analysis confirms that the reported
distributional fit is insensitive to this choice, as the highest-likelihood
distributional solution found on every split is the one reached by
the default start. The single-answer regime has an additional
higher-likelihood degenerate mode (\Cref{app:multistart}).
We fit $\pi$ on unlabelled training configurations and freeze it before
evaluation; no gold label is used during fitting.

\subsection{What the rule actually is}
\label{sec:rule}

\Cref{prop:reduce,prop:pool} characterise how \eqref{eq:verdict} reduces in
the two evidence regimes.

\begin{proposition}[Single answers: the log-odds weighted vote]
\label[proposition]{prop:reduce}
With one answer per expert and $m \ge 1$,
\begin{align*}
  \sum_e \log P(a_e \mid z = v) &\;=\; C \;+\!\! \sum_{e\,:\,a_e = v}\!\! w_e, \\
  w_e &\;=\; \log \frac{\pi_e\, m}{1 - \pi_e},
\end{align*}
with $C = \sum_e \log\frac{1-\pi_e}{m}$ independent of $v$. Hence
\eqref{eq:verdict} coincides with the weighted vote \eqref{eq:vote}.
\end{proposition}
\begin{proof}
Substitute \eqref{eq:emission} and add and subtract $\sum_e \log\frac{1-\pi_e}{m}$.
\end{proof}

\Cref{prop:reduce} says that, with single answers, the jury is the
Nitzan-Paroush log-odds rule \citep{nitzan1982optimal} in the $K$-ary form of
\citet{beyondmajority2025}, with $\pi$ estimated unsupervised instead of read off
labels.

\begin{proposition}[Distributions: a signed linear fusion]
\label[proposition]{prop:pool}
Let each expert supply a distribution $p_e$ over $\cand$ and let its contribution
to the log-likelihood be $\alpha \sum_{u \in \cand} p_e(u) \log P(a_e = u \mid z = v)$,
with $\alpha > 0$ the evidence weight of \Cref{sec:ablations}. Then
\begin{align*}
  \hat{y} &\;=\; \operatorname*{arg\,max}_{v \in \cand} \; \sum_{e} w_e \, p_e(v), \\
  w_e &\;=\; \log \frac{\pi_e\, m}{1 - \pi_e}.
\end{align*}
\end{proposition}
\begin{proof}
$\sum_u p_e(u) \log P(a_e = u \mid z=v) = p_e(v)\log\pi_e + (1 - p_e(v))
\log\frac{1-\pi_e}{m} = C_e + p_e(v)\, w_e$ with $C_e$ independent of $v$; sum
over $e$, drop the constant, and drop the common factor $\alpha$, which leaves
the argmax unchanged.
\end{proof}

% The score of \Cref{prop:pool} is an inner product between a weight vector and
% the experts' distributions, with the same log-odds weights as
% \Cref{prop:reduce}. When every $w_e \ge 0$ and at least one $w_e > 0$,
% dividing by $\sum_e w_e$ --- which leaves the argmax unchanged --- makes it the
% linear opinion pool of \citet{stone1961opinion}, whose value is a distribution
% over $\cand$. Log-odds weights need not be non-negative, so in general the
% object is a \emph{signed} linear fusion whose score is a ranking function and
% not a probability, normalised or otherwise. The posterior $\gamma$ of
% \eqref{eq:verdict} is a distribution either way --- it is the score that need
% not be.

\begin{remark}
\Cref{prop:pool} keeps the rule within weighted linear fusion. It remains an
inner product between a weight vector and per-candidate
evidence. Unlike voting over \emph{argmax} answers, however, a candidate can win
without being any expert's top choice, so the reachable set can be strictly
larger. Together the two propositions show that the method's learned content is
the reliability vector $\pi$, which determines the expert-specific component of
$w$, making the same fusion rule with label-fitted weights the natural
supervised baseline.
\end{remark}

\begin{remark}
The weights are not entirely fixed in advance. Since $m = |\cand| - 1$ depends on
the answers observed, $w_e = \log\frac{\pi_e}{1-\pi_e} + \log m$ carries a
question-level offset shared by every expert, contributing
$\log m \sum_e p_e(v)$ to the score --- a multiple of the equal-weight pool. Only
the per-expert part is fixed before the answers are seen. Note also that $w_e < 0$
whenever $\pi_e < 1/|\cand|$, in which case the expert is voted against rather
than merely ignored. The fusion score ranks candidates but need not be a probability,
the posterior $\gamma_q$ of \eqref{eq:verdict} remains a probability distribution either way.
\end{remark}

\Cref{fig:jury} puts these pieces together on a toy example: the emission model,
the label-free fit of $\pi$, the log-odds weights and the signed linear fusion.

\begin{figure}[tb]
  \centering
  \resizebox{0.9\textwidth}{!}{%
    % Input-ready TikZ figure; no standalone document wrapper.
% Main-document preamble: \usepackage{tikz}
%                         \usetikzlibrary{arrows.meta}

\begingroup
\definecolor{ink}{HTML}{222222}
\definecolor{muted}{HTML}{555555}
\definecolor{rule}{HTML}{B8B6B1}
\definecolor{warm}{HTML}{F3F1EC}
\definecolor{juryink}{HTML}{4B428B}
\definecolor{juryfill}{HTML}{F1EFFA}
\definecolor{wrongink}{HTML}{8D3434}
\definecolor{wrongfill}{HTML}{FBEFEF}
\definecolor{correctink}{HTML}{426430}
\definecolor{correctfill}{HTML}{EFF5E8}
\newcommand{\juryfigfont}{\fontsize{8.6}{10.6}\selectfont}
\newcommand{\jurysmallfigfont}{\fontsize{8.1}{9.8}\selectfont}
\newcommand{\jurytitlefont}{\fontsize{9.6}{11.6}\selectfont\bfseries}

\begin{tikzpicture}[
  x=1cm,y=1cm,
  every node/.style={inner sep=0pt,outer sep=0pt,text=ink,font=\juryfigfont},
  panel/.style={draw=rule,line width=.45pt,rounded corners=2.5pt,fill=white},
  flow/.style={-{Stealth[length=3.3pt,width=3.3pt]},draw=muted,line width=.65pt}
]
  \path[use as bounding box] (0,0) rectangle (17.4,4.5);

  % A shared input, the choice of weights, and a shared scoring rule.
  \draw[panel] (0.02,0.14) rectangle (4.92,4.36);
  \draw[panel] (5.52,0.14) rectangle (11.88,4.36);
  \draw[panel] (12.48,0.14) rectangle (17.38,4.36);
  \draw[flow] (4.98,2.25) -- (5.44,2.25);
  \draw[flow] (11.94,2.25) -- (12.40,2.25);

  % Same answer distributions: bold marks each expert's top answer.
  \node[font=\jurytitlefont] at (2.47,4.03) {Same answer distributions};
  \node[text=muted,font=\jurysmallfigfont] at (2.47,3.57)
    {One question; five experts};

  \draw[draw=ink,line width=.55pt] (0.24,3.31) -- (4.70,3.31);
  \draw[draw=ink,line width=.35pt] (0.24,2.89) -- (4.70,2.89);
  \draw[draw=ink,line width=.55pt] (0.24,0.71) -- (4.70,0.71);
  \node at (0.70,3.10) {Expert};
  \node at (1.81,3.10) {$p_e(A)$};
  \node at (2.95,3.10) {$p_e(B)$};
  \node at (4.10,3.10) {$p_e(C)$};
  \foreach \e/\y/\a/\b/\c in {
    1/2.69/{\textbf{0.70}}/{0.20}/{0.10},
    2/2.25/{\textbf{0.60}}/{0.25}/{0.15},
    3/1.81/{0.25}/{\textbf{0.55}}/{0.20},
    4/1.37/{0.15}/{\textbf{0.60}}/{0.25},
    5/0.93/{0.20}/{\textbf{0.50}}/{0.30}}{
    \node at (0.70,\y) {\e};
    \node at (1.81,\y) {\a};
    \node at (2.95,\y) {\b};
    \node at (4.10,\y) {\c};
  }
  \node[text=muted,font=\jurysmallfigfont] at (2.47,0.44)
    {Bold: each expert's top answer};

  % Only the weighting changes.
  \node[font=\jurytitlefont] at (8.70,4.03) {Only the weights change};
  \draw[draw=none,fill=warm,rounded corners=2pt]
    (5.70,3.22) rectangle (11.70,3.77);
  \node at (8.70,3.49) {\textbf{Electorate:} equal weights, $w_e=1$};

  \draw[draw=none,fill=juryfill,rounded corners=2pt]
    (5.70,0.32) rectangle (11.70,3.04);
  \node[text=juryink,font=\juryfigfont\bfseries] at (8.70,2.80) {Jury: EM, no labels};
  \node[text=muted,font=\jurysmallfigfont] at (8.70,2.41)
    {Fit on unlabelled training data, then freeze $\pi$};
  \node at (8.70,1.93)
    {$w_e=\log\!\dfrac{\pi_e\,m}{1-\pi_e},\qquad m=2$};

  \node at (6.35,1.42) {Expert};
  \node at (6.35,0.94) {$\pi_e$};
  \node at (6.35,0.54) {$w_e$};
  \draw[draw=juryink,line width=.35pt] (5.99,1.21) -- (11.40,1.21);
  \foreach \e/\x/\rel/\weight in {
    1/7.40/{0.75}/{1.79},
    2/8.32/{0.70}/{1.54},
    3/9.24/{0.40}/{0.29},
    4/10.16/{0.35}/{0.07},
    5/11.08/{0.20}/{$\mathbf{-0.69}$}}{
    \node at (\x,1.42) {\e};
    \node at (\x,0.94) {\rel};
    \node at (\x,0.54) {\weight};
  }

  % The same fusion rule, with the two sets of weights.
  \node[font=\jurytitlefont] at (14.93,4.03) {Same linear fusion};
  \node at (14.93,3.48)
    {$\mathrm{score}(v)=\sum_e w_e\,p_e(v)$};

  \draw[draw=ink,line width=.55pt] (12.80,3.17) -- (17.06,3.17);
  \draw[draw=ink,line width=.35pt] (12.80,2.71) -- (17.06,2.71);
  \draw[draw=ink,line width=.55pt] (12.80,1.48) -- (17.06,1.48);
  \node at (13.49,2.93) {Candidate};
  \node at (15.14,2.93) {Electorate};
  \node at (16.50,2.93) {Jury};
  \foreach \candidate/\y/\equal/\jury in {
    A/2.48/{1.90}/{\textbf{2.12}},
    B/2.08/{\textbf{2.10}}/{0.60},
    C/1.68/{1.00}/{0.28}}{
    \node at (13.49,\y) {\candidate};
    \node at (15.14,\y) {\equal};
    \node at (16.50,\y) {\jury};
  }
  \draw[draw=wrongink,line width=.4pt,fill=wrongfill,rounded corners=2pt]
    (12.70,0.95) rectangle (17.16,1.36);
  \node[text=wrongink] at (14.93,1.155)
    {\textbf{Electorate: B}\quad wrong};
  \draw[draw=correctink,line width=.4pt,fill=correctfill,rounded corners=2pt]
    (12.70,0.35) rectangle (17.16,0.76);
  \node[text=correctink] at (14.93,0.555)
    {\textbf{Jury: A}\quad correct};

\end{tikzpicture}

\endgroup%
  }
\caption{Toy example with five experts and three candidates. The gold answer is A, but three experts favour B. Both rules use the same answer distributions; only the weights differ. The equal-weight pool selects B, while the jury estimates $\pi$ from unlabelled training configurations using \eqref{eq:mstep-dist} under \eqref{eq:emission}, freezes it before test time, and selects A through the same linear fusion (\Cref{prop:pool}). The numbers are illustrative.}
  \label{fig:jury}
\end{figure}
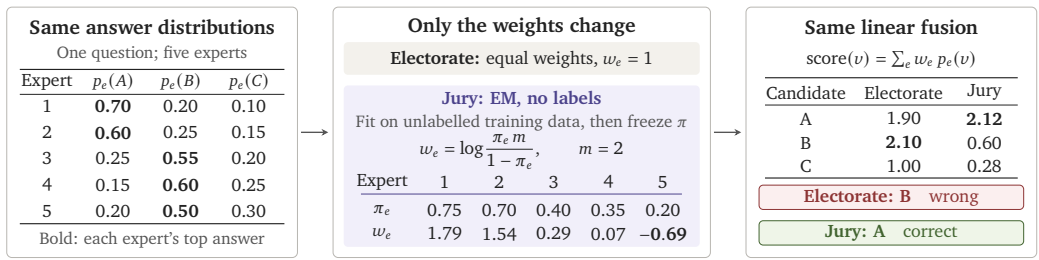

% \begin{figure}[t]
%   \centering
%   \input{jury_vs_electorate_horizontal.tex}
%   \caption{A toy example in which most experts favour the wrong answer. Five experts
%   spread their probability over three candidates: the gold answer is A, but three of
%   them favour B. Both columns read the same distributions and differ only in where
%   the weights come from. The poll (left) gives every expert the same weight,
%   and the pool follows the majority to B. The jury (right) models the answers by \eqref{eq:emission} and fits $\pi$ by
%   \eqref{eq:mstep-dist} on unlabelled training configurations, frozen
%   before this question is seen. It turns $\pi$ into log-odds weights, negative only
%   for the expert below $1/|\cand|$, and ranks the candidates by
%   $\sum_e w_e\,p_e(v)$ (\Cref{prop:pool}). The numbers are illustrative.}
%   \label{fig:jury}
% \end{figure}

\subsection{Modelling choices}
\label{sec:ablations}

\paragraph{Evidence weight.}
Let $p_{e,q}$ be the empirical distribution obtained from $n$ samples of expert
$e$. We weight its log-likelihood contribution by $\alpha$. Setting
$\alpha = n$ recovers the multinomial log-likelihood, up to
a term independent of the latent answer and the reliabilities. We therefore use
$\alpha = n$ as the default: $\alpha = 8$ for the sampled distributions, since
the cache holds eight samples per expert and question (\Cref{sec:setup}), and
$\alpha = 1$ for a single greedy answer. As an ablation, we also consider
$\alpha = 1$ for a sampled distribution, which treats the whole distribution as
one unit of evidence (see \Cref{app:ablations}). At fixed $\pi$, $\alpha$ does not change the argmax in
\Cref{prop:pool}, because it multiplies every candidate score by the same
positive constant. It does, however, change the E-step distribution $\gamma$,
and therefore the fitted reliabilities.

\Cref{eq:emission} assumes that experts are conditionally independent given the
latent answer. To test sensitivity to shared failure modes, we also consider an
extension with a second latent variable per question, a \emph{lure}: a wrong
answer toward which a specified group of experts may be jointly drawn. Each
expert has a susceptibility parameter controlling its response to the lure. We
consider groups defined by pre-training lineage and a pool-wide group (see \Cref{app:ablations}); the
independent model remains the default.

% ===========================================================================
\section{Results}
\label{sec:results}
% ===========================================================================

This section tests the jury empirically.
\Cref{sec:setup} describes the benchmark and the evaluation setup.
\Cref{sec:comparison} compares the jury with label-free and supervised baselines and the best single expert, separates the gains from answer distributions and expert weighting, and reports robustness and compute.
\Cref{sec:ablresults} examines the modelling choices of \Cref{sec:ablations} and calibration.
Further experimental details and analyses are in \Cref{app:setup,app:robustness,app:ablations}.

\subsection{Setup}
\label{sec:setup}

\paragraph{Pool.}
Our pool contains nine instruction-tuned experts of up to 3B parameters, selected to vary in size, domain, and model
family. The goal is to include models that fail in different ways, giving the combiner useful disagreement to exploit. \texttt{qwen3b} is the strongest expert overall ($75.0\%$) and on multiple choice, while \texttt{olmo2-1b} leads on GSM8K despite being one of the weakest experts overall. The full model list and individual accuracies are given in \Cref{app:setup}, \Cref{tab:experts}.

\paragraph{Data.}
The benchmark has $1560$ questions from four sources: $240$ from MMLU
\cite{hendrycks2021mmlu} ($40$ from each of six subjects spanning humanities,
law and science), $120$ GSM8K \cite{cobbe2021gsm8k} problems, the only
generative slice, $600$ ARC-Challenge \cite{clark2018arc} and $600$ HellaSwag
\cite{zellers2019hellaswag} questions. Every source is scored $0/1$ against a
single gold answer %, which the model of \eqref{eq:emission} requires
(for more details see \Cref{app:setup}).

\paragraph{Answers and repeated sampling.}
Multiple choice is scored in one forward pass from the log-probabilities of the
four option letters; GSM8K is answered by chain-of-thought generation, from
which the final number is extracted (\Cref{app:setup}). Besides one greedy answer
per expert, the cache holds eight samples per expert per question ($T = 0.7$,
$p = 0.95$), \ie self-consistency sampling \cite{wang2023selfconsistency}
applied per expert. These samples give the distribution $p_e$ that
\Cref{prop:pool} consumes. On multiple choice they are draws from a
distribution that the single forward pass has already computed, so they cost
nothing extra.

\paragraph{Protocol and statistics.}
Every method is evaluated over ten stratified $50/50$ train/test splits of the same
cached answers; methods that need labels see only the training half. The jury is
fitted on the unlabelled answer configurations of the same training half;
fitting it on all configurations would leak no label, but would not be a matched
comparison. We report the
mean and standard deviation over splits. The ten test sets overlap heavily, so a
paired $t$-test over splits would overstate the evidence
\cite{dietterich1998tests,nadeau2003inference}. All comparisons therefore use a
paired bootstrap over \emph{questions} with $20{,}000$ resamples, where a
question's score for a method is its mean correctness over the splits in which it
was held out. Compute is the average number of parameters activated per question,
which charges a router for the one expert it calls and a fusion rule for all nine.

\paragraph{The supervised linear fusion.}
The supervised linear fusion uses the fusion rule of \Cref{prop:pool}
with its reliabilities read off labels. Its reliability for expert $e$
is the accuracy of $e$'s greedy answer on the training half. The
comparison with the jury is between two ways of estimating the
reliabilities of one rule, not against every way of fitting weights
from labels. The two estimators also differ in three implementation
conventions (see \Cref{app:supervised}).

\subsection{The comparison}
\label{sec:comparison}

\Cref{tab:main} is the main comparison: ten post-hoc combiners re-implemented on this cache and run on the same ten splits with the same experts.
%\Cref{tab:main} is the main comparison: ten post-hoc combiners re-implemented on
%this cache and run on the same ten splits with the same experts.
%\Cref{fig:frontier} plots accuracy against compute.

\begin{table}[tb]
\centering
\small
\setlength{\tabcolsep}{3.0pt}
\resizebox{\linewidth}{!}{\begin{tabular}{>{\raggedright\arraybackslash}p{5.2cm}cccrrr}
\toprule
Method & Labels & Params (B) & Accuracy & Plurality & Minority & $\Delta$ vs.\ jury~[95\% CI] \\
\midrule
Best single expert & -- & -- & 75.0 $\pm$ 1.0 & -- & -- & +2.0$^{*}$~[+0.5, +3.6] \\
\quad + self-consistency on GSM8K & -- & -- & 75.3 $\pm$ 1.1 & -- & -- & +1.6$^{*}$~[+0.0, +3.1] \\
\midrule
Majority vote~\cite{li2024moreagents} & no & 14.4 & 63.0 $\pm$ 1.6 & 100.0 & 6.2 & +14.1$^{*}$~[+12.0, +16.3] \\
Inverse Surprising Popularity~\cite{prelec2017surprising,beyondmajority2025} & no & 14.4 & 69.1 $\pm$ 1.6 & 98.5 & 27.4 & +7.9$^{*}$~[+6.2, +9.6] \\
Smoothie-Global~\cite{guha2024smoothie} & no & 1.5 & 66.3 $\pm$ 1.3 & 85.3 & 43.1 & +10.5$^{*}$~[+8.6, +12.4] \\
Smoothie-Local~\cite{guha2024smoothie} & no & 2.0 & 64.2 $\pm$ 1.6 & 83.4 & 40.3 & +12.7$^{*}$~[+11.1, +14.4] \\
Optimal Weight~\cite{nitzan1982optimal,beyondmajority2025} & yes & 14.4 & 73.0 $\pm$ 1.4 & 95.9 & 43.9 & +3.9$^{*}$~[+2.7, +5.2] \\
LLM-TOPLA~\cite{tekin2024topla} & yes & 4.5 & 73.9 $\pm$ 1.5 & 89.3 & 58.5 & +2.9$^{*}$~[+1.4, +4.5] \\
MoRE~\cite{si2023more} & yes & 14.4 & 74.6 $\pm$ 1.2 & 93.2 & 53.4 & +2.4$^{*}$~[+1.4, +3.6] \\
MoRE + repeated sampling~\cite{si2023more} & yes & 22.2 & 76.4 $\pm$ 1.3 & 93.3 & 58.5 & +0.5~[-0.4, +1.5] \\
Equal-weight pool~\cite{stone1961opinion} & no & 22.2 & 71.4 $\pm$ 1.8 & 96.8 & 37.2 & +5.5$^{*}$~[+4.1, +7.0] \\
Supervised linear fusion~\cite{stone1961opinion,beyondmajority2025} & yes & 22.2 & 76.2 $\pm$ 1.4 & 95.6 & 53.5 & +0.6$^{*}$~[+0.1, +1.2] \\
\midrule
Answer scoring, discriminative & yes & 14.4 & 76.2 $\pm$ 1.3 & 93.1 & 58.7 & +0.7~[-0.4, +1.9] \\
\quad + repeated sampling & yes & 22.2 & 76.7 $\pm$ 1.1 & 92.9 & 60.3 & +0.1~[-1.0, +1.2] \\
\textbf{Jury (ours), greedy answers} & \textbf{no} & 14.4 & 72.3 $\pm$ 1.1 & 94.0 & 45.3 & +4.6$^{*}$~[+3.3, +5.9] \\
\textbf{\quad + repeated sampling} & \textbf{no} & 22.2 & 76.9 $\pm$ 1.3 & 95.7 & 55.5 &  \\
\midrule
Oracle & -- & -- & 94.3 $\pm$ 0.5 & -- & -- & \\
\bottomrule
\end{tabular}
}
\caption{Methods evaluated on the same pool, cache, and ten splits.
With \emph{self-consistency on GSM8K}, the best single expert predicts by taking the mode of the same eight samples used by the distributional methods.
The equal-weight pool and supervised linear fusion use answer
distributions; the latter fits log-odds weights from labels
(\Cref{prop:pool}). \emph{Plurality} denotes questions where the gold answer is the unique plurality winner; \emph{Minority} denotes those where at least one expert gives the gold answer but it is not the unique plurality winner, including ties. $\Delta$ is the paired difference from the jury with repeated
sampling, with a $95\%$ bootstrap interval; $^{*}$ marks $p<0.05$. Accuracies
are means over the ten splits, whereas $\Delta$ averages per question first, so
the two can differ by a few tenths. Implementation details in supplementary
material.}
%\caption{Post-hoc combiners on the same pool, cache and ten splits.
%The equal-weight pool and the supervised linear fusion are both over the answer
%distributions. The first is a linear opinion pool in Stone's sense; the second
%has log-odds weights fitted on labels, which go negative here and make it a
%signed fusion (\Cref{prop:pool}). \emph{Plurality} and \emph{Minority} report accuracy on question
%subsets defined by the experts' greedy answers: for \emph{Plurality},
%the gold answer is the unique plurality winner; for \emph{Minority},
%at least one expert proposes it, but it is not the unique plurality
%winner (including ties for first place). $\Delta$ is the paired difference against the jury with
%repeated sampling, with a $95\%$ bootstrap interval over questions; $^{*}$ marks
%$p < 0.05$. Implementation adaptations required to apply published methods to
%our benchmark are documented in \texttt{scripts/aggregator\_vs\_literature.py}
%in the supplementary material.}
\label{tab:main}
\end{table}

\paragraph{Two separable gains.}
Majority voting is the simplest form of model aggregation:
every expert casts one equally weighted vote. Over all nine experts it scores
$63.0 \pm 1.6\%$, twelve points below the best single expert. Replacing the
argmax answers with the answer distributions, at equal weights, gives
$71.4 \pm 1.8\%$: a $+8.4$-point gain from keeping the distributions, with no
weights and no labels. Weighting the experts by log-odds adds a further $+5.5$
points, reaching $76.9 \pm 1.3\%$. The first gain is available to any method;
the second is where the jury operates, since it must learn from the testimony
alone which witnesses to trust.

\Cref{tab:decomp} shows that the first gain comes from the distribution
not the extra samples. Its middle row is a sample-matched control that
uses the same eight samples per expert, at the same cost, but reduces them to
their mode before the rule sees them.

The two effects are separable but interact:
weighting improves every evidence regime, and retaining the full distribution
rather than its sample-matched mode improves every weighting scheme, by $2.2$
to $8.3$ points.

\begin{table}[tb]
\centering
\small
\resizebox{\linewidth}{!}{\begin{tabular}{lrrrr}
\toprule
& & \multicolumn{3}{c}{Weights over the experts} \\
\cmidrule(lr){3-5}
Evidence the rule is given & Params (B) & Equal & Log-odds, supervised & Log-odds, label-free \\
\midrule
Greedy answer, 1 per expert & 14.4 & 63.0 $\pm$ 1.6 & 73.0 $\pm$ 1.4 & 72.3 $\pm$ 1.1 \\
Mode of 8 samples & 22.2 & 63.1 $\pm$ 1.8 & 73.4 $\pm$ 1.4 & 74.7 $\pm$ 1.4 \\
Full 8-sample distribution & 22.2 & 71.4 $\pm$ 1.8 & 76.2 $\pm$ 1.4 & 76.9 $\pm$ 1.3 \\
\midrule
\emph{Distribution minus mode} & & +8.3 & +2.8 & +2.2 \\
\emph{Mode minus greedy} & & +0.1 & +0.4 & +2.3 \\
\bottomrule
\end{tabular}
}
\caption{What the rule is allowed to see, against how it weights the experts. The
middle row is the sample-matched control: it pays for the eight samples in full
and then discards everything but their mode. Every entry is the mean over the same
ten splits, $\pm$ one standard deviation. Differences are computed before
rounding.}
\label{tab:decomp}
\end{table}

\paragraph{Label-free weights match supervised ones.}
The supervised linear fusion has the same functional form as the jury, with
weights fitted on the labels of the training half. It reaches
$76.2 \pm 1.4\%$, and the jury is $+0.6$ points above it
($[+0.1, +1.2]$, $p = 0.02$). Conditional on the ten fitted models the
difference is small but significant in the jury's favour. Once the variability
of refitting is included as well, the nested bootstrap of \Cref{app:refit} gives
a difference of $+0.6$ points with a $95\%$ interval of $[-0.6, +2.0]$:
no statistically significant difference, and an interval that admits
differences of either sign, so not a demonstration of equivalence. On this benchmark, against this baseline, we find no evidence
that supervised weight estimation outperforms label-free estimation.

Further supervised comparisons and per-task results are reported in
\Cref{app:supervised,app:refit,app:spec}.

\paragraph{Label-free comparisons.}
On the full benchmark, the greedy jury scores $72.3\%$, versus $69.1\%$ for ISP
\cite{beyondmajority2025}, and with distributions it reaches $76.9\%$, $+5.5$
over the equal-weight pool ($[+4.1, +7.0]$) and $+7.9$ over ISP
($[+6.2, +9.6]$). On the $1440$ fixed-label questions, SkillAggregation
\cite{sun-etal-2025-skillaggregation} reaches $73.2\%$ with distributions
without a labelled development set; even
its test-selected upper reference, SkillAggregation-X, reaches $74.0\%$, $2.7$
points below the jury at $76.6\%$. OW-L and OW-I \cite{beyondmajority2025} reach $71.3\%$
and $71.6\%$ on greedy answers, versus $72.4\%$ for the greedy jury, with
intervals including zero (\Cref{app:skillagg}). Most of the gain comes from
the minority stratum, where the gold answer is not the unique plurality winner: accuracy
rises from $6.2\%$ for majority voting to $55.5\%$ for the jury, while the jury
also reaches $95.7\%$ on plurality questions.

\paragraph{Best single expert.}
\texttt{qwen3b} is the strongest expert overall ($75.0\%$), and the jury
reaches $76.9\%$, $+2.0$ ($[+0.5, +3.6]$, $p = 0.012$). Giving every expert the
same eight-sample self-consistency on GSM8K and selecting on the training half
again chooses \texttt{qwen3b} on every split, at $75.3\%$; the jury remains
$+1.6$ ahead ($[+0.0, +3.1]$, $p = 0.048$). The gain is concentrated on GSM8K,
where \texttt{qwen3b} is weakest and \texttt{olmo2-1b} is the strongest expert;
there the jury gains $20$ points, while on the sources \texttt{qwen3b} already
dominates it essentially matches that expert (\Cref{app:bysource}).

\paragraph{Robustness.}
Nested resampling leaves the main conclusions unchanged: the jury remains
above the equal-weight pool ($+5.5$, $[+2.9, +8.3]$) and majority vote
($+14.1$, $[+10.5, +17.8]$), while its $+0.6$ margin over the supervised linear
fusion is not significant once refitting is resampled as well ($[-0.6, +2.0]$;
\Cref{app:refit}).

\paragraph{Across sub-pools.}
Across $246$ sub-pools, the jury beats the equal-weight pool on $79\%$, $98\%$
and $100\%$ of the three, five and seven expert pools, and the best expert on
$82\%$, $96\%$ and $100\%$, respectively; the advantage is therefore not
specific to the full nine-expert pool. \Cref{app:pools} gives the full
breakdown and the $\alpha = 1$ degeneracies.

\paragraph{Cost.}
Sampled variants activate $22.2$B parameters per question versus $14.4$B for
greedy variants, because only GSM8K re-runs the models; \Cref{fig:frontier} in
the appendix plots accuracy against cost.

\subsection{Model checks and calibration}
\label{sec:ablresults}

\Cref{app:ablations} reports the full checks. Experts agree on wrong answers
more often than independence predicts, but the correlated-error extension
slightly lowers accuracy, and selective override does not help. The posterior
is strongly overconfident (mean top-1 confidence $0.991$ versus $76.9\%$
accuracy; expected calibration error, computed over ten equal-width confidence bins, is 0.223) and requires recalibration before use as a confidence
score.
%Expected calibration error (ECE), computed over ten equal-width confidence bins, is 0.223.

\section{Conclusion}
% ===========================================================================

We asked whether an ensemble of language models can work like a jury: infer
credibility from testimony rather than be told it in advance.
Keeping experts' answer distributions instead of only their argmax improves majority
voting by $8.4$ points, and weighting experts adds a further $5.5$ points.
Our contribution is to the second: showing that a simple generative model \citep{dawid1979em}, fitted by EM on unlabelled answers,
learns expert weights without reading a gold answer.

The resulting jury reaches $76.9 \pm 1.3\%$, beating the pool's best greedy expert by $2.0$ points, the same expert with eight-sample self-consistency on GSM8K by $1.6$ points, and the
strongest label-free alternative on the full benchmark by $5.5$ points, while performing comparably to the supervised linear fusion, the same rule with weights fitted from labels, with a small edge that is significant conditional on the
fitted models but not once the variability of refitting is included. The method therefore provides a label-free estimator of expert
weights for linear fusion.

% On this pool, supervising the expert weights provides no measurable accuracy advantage, and the same qualitative
% pattern holds on average across the sub-pools of five or more experts drawn from it. The result suggests a
% simple alternative to learning whom to trust from an answer key: let the jury infer it from the testimony itself.

\section{AI usage}
In this work, we used generative AI tools (Codex, Claude) to assist with: \textbf{(i)} checking algebraic steps in the proofs, after the propositions and proof strategies had been developed by the authors; \textbf{(ii)} supporting the implementation of the evaluation code; \textbf{(iii)} editing the manuscript to improve readability; and \textbf{(iv)} identifying relevant literature.

% *ACL papers carry an unnumbered Limitations section after the conclusion
% and before the references.
\section*{Limitations}

Our evaluation uses a fixed pool of nine models and 1,560 questions, mostly multiple-choice.
We test every subset of three, five, and seven experts, providing extensive coverage of combinations
within this pool, but not of different underlying models.
Most of the gain over the best single expert comes from GSM8K.
The absence of a significant supervised advantage under the nested bootstrap does not establish
equivalence or generalisation beyond the evaluated setting.

The jury requires outputs from every expert and repeated sampling for generative tasks,
increasing inference cost. The comparison with the best greedy expert therefore measures the benefit
of the complete system, including additional sampling, rather than aggregation alone.

\begingroup
\raggedright
\setlength{\bibsep}{1pt plus 1pt}
\bibliographystyle{plainnat}
\bibliography{jury}

\begin{thebibliography}{54}
\providecommand{\natexlab}[1]{#1}
\providecommand{\url}[1]{\texttt{#1}}
\expandafter\ifx\csname urlstyle\endcsname\relax
  \providecommand{\doi}[1]{doi: #1}\else
  \providecommand{\doi}{doi: \begingroup \urlstyle{rm}\Url}\fi

\bibitem[Ai et~al.(2026)Ai, Pan, Simchi-Levi, Tambe, and Xu]{beyondmajority2025}
Rui Ai, Yuqi Pan, David Simchi-Levi, Milind Tambe, and Haifeng Xu.
\newblock Beyond majority voting: {LLM} aggregation by leveraging higher-order information.
\newblock In \emph{Proceedings of the 43rd International Conference on Machine Learning}, volume 306 of \emph{Proceedings of Machine Learning Research}, pages 1311--1340. PMLR, 2026.
\newblock URL \url{https://proceedings.mlr.press/v306/ai26d.html}.

\bibitem[Chen et~al.(2026{\natexlab{a}})Chen, Xun, Zhou, Qi, Zhang, Zhang, Chen, Hu, Qu, and Hu]{modelswitch2025}
Jianhao Chen, Zishuo Xun, Bocheng Zhou, Han Qi, Hangfan Zhang, Qiaosheng Zhang, Yang Chen, Wei Hu, Yuzhong Qu, and Shuyue Hu.
\newblock Do we truly need so many samples? multi-llm repeated sampling efficiently scales test-time compute.
\newblock In \emph{Proceedings of the AAAI Conference on Artificial Intelligence}, volume~40, pages 20083--20091, 2026{\natexlab{a}}.
\newblock \doi{10.1609/aaai.v40i24.39094}.

\bibitem[Chen et~al.(2024)Chen, Zaharia, and Zou]{chen2023frugalgpt}
Lingjiao Chen, Matei Zaharia, and James Zou.
\newblock Frugal{GPT}: How to use large language models while reducing cost and improving performance.
\newblock \emph{Transactions on Machine Learning Research}, 2024.
\newblock ISSN 2835-8856.
\newblock URL \url{https://openreview.net/forum?id=cSimKw5p6R}.
\newblock Featured Certification.

\bibitem[Chen et~al.(2026{\natexlab{b}})Chen, Lu, Li, Chen, Li, Sun, Luo, Mao, Li, Xiao, Yang, Ban, and Sun]{chen2025llmensemble}
Zhijun Chen, Xiaodong Lu, Jingzheng Li, Pengpeng Chen, Zhuoran Li, Kai Sun, Yuankai Luo, Qianren Mao, Ming Li, Likang Xiao, Dingqi Yang, Yikun Ban, and Hailong Sun.
\newblock Harnessing multiple large language models: A survey on llm ensemble.
\newblock In Diego Calvanese, editor, \emph{Proceedings of the Thirty-Fifth International Joint Conference on Artificial Intelligence, {IJCAI-26}}, pages 7807--7816. International Joint Conferences on Artificial Intelligence Organization, 8 2026{\natexlab{b}}.
\newblock \doi{10.24963/ijcai.2026/867}.
\newblock URL \url{https://doi.org/10.24963/ijcai.2026/867}.
\newblock Survey Track.

\bibitem[Clark et~al.(2018)Clark, Cowhey, Etzioni, Khot, Sabharwal, Schoenick, and Tafjord]{clark2018arc}
Peter Clark, Isaac Cowhey, Oren Etzioni, Tushar Khot, Ashish Sabharwal, Carissa Schoenick, and Oyvind Tafjord.
\newblock Think you have solved question answering? try arc, the ai2 reasoning challenge, 2018.
\newblock URL \url{https://arxiv.org/abs/1803.05457}.

\bibitem[Clemen and Winkler(1999)]{clemen1999combining}
Robert~T. Clemen and Robert~L. Winkler.
\newblock {Combining Probability Distributions from Experts in Risk Analysis}.
\newblock \emph{Risk Analysis}, 19\penalty0 (2):\penalty0 187--203, 1999.
\newblock \doi{10.1111/j.1539-6924.1999.tb00399.x}.

\bibitem[Cobbe et~al.(2021)Cobbe, Kosaraju, Bavarian, Chen, Jun, Kaiser, Plappert, Tworek, Hilton, Nakano, Hesse, and Schulman]{cobbe2021gsm8k}
Karl Cobbe, Vineet Kosaraju, Mohammad Bavarian, Mark Chen, Heewoo Jun, Lukasz Kaiser, Matthias Plappert, Jerry Tworek, Jacob Hilton, Reiichiro Nakano, Christopher Hesse, and John Schulman.
\newblock Training verifiers to solve math word problems, 2021.
\newblock URL \url{https://arxiv.org/abs/2110.14168}.

\bibitem[Dawid and Skene(1979)]{dawid1979em}
A.~P. Dawid and A.~M. Skene.
\newblock {Maximum Likelihood Estimation of Observer Error-Rates Using the EM Algorithm}.
\newblock \emph{Journal of the Royal Statistical Society: Series C (Applied Statistics)}, 28\penalty0 (1):\penalty0 20--28, 1979.
\newblock \doi{10.2307/2346806}.

\bibitem[DeGroot(1974)]{degroot1974consensus}
Morris~H. DeGroot.
\newblock {Reaching a Consensus}.
\newblock \emph{Journal of the American Statistical Association}, 69\penalty0 (345):\penalty0 118--121, 1974.
\newblock \doi{10.1080/01621459.1974.10480137}.

\bibitem[Dekoninck et~al.(2025)Dekoninck, Baader, and Vechev]{cascaderouting2024}
Jasper Dekoninck, Maximilian Baader, and Martin Vechev.
\newblock A unified approach to routing and cascading for {LLM}s.
\newblock In Aarti Singh, Maryam Fazel, Daniel Hsu, Simon Lacoste-Julien, Felix Berkenkamp, Tegan Maharaj, Kiri Wagstaff, and Jerry Zhu, editors, \emph{Proceedings of the 42nd International Conference on Machine Learning}, volume 267 of \emph{Proceedings of Machine Learning Research}, pages 12987--13010. PMLR, 13--19 Jul 2025.
\newblock URL \url{https://proceedings.mlr.press/v267/dekoninck25a.html}.

\bibitem[Dietterich(1998)]{dietterich1998tests}
Thomas~G. Dietterich.
\newblock {Approximate Statistical Tests for Comparing Supervised Classification Learning Algorithms}.
\newblock \emph{Neural Computation}, 10\penalty0 (7):\penalty0 1895--1923, 1998.
\newblock \doi{10.1162/089976698300017197}.

\bibitem[Ding et~al.(2024)Ding, Mallick, Wang, Sim, Mukherjee, R\"{u}hle, Lakshmanan, and Awadallah]{ding2024hybrid}
Dujian Ding, Ankur Mallick, Chi Wang, Robert Sim, Subhabrata Mukherjee, Victor R\"{u}hle, Laks V.~S. Lakshmanan, and Ahmed Awadallah.
\newblock Hybrid llm: Cost-efficient and quality-aware query routing.
\newblock In B.~Kim, Y.~Yue, S.~Chaudhuri, K.~Fragkiadaki, M.~Khan, and Y.~Sun, editors, \emph{International Conference on Learning Representations}, volume 2024, pages 41348--41366, 2024.
\newblock URL \url{https://proceedings.iclr.cc/paper_files/paper/2024/file/b47d93c99fa22ac0b377578af0a1f63a-Paper-Conference.pdf}.

\bibitem[Feng et~al.(2025)Feng, Shen, and You]{feng2024graphrouter}
Tao Feng, Yanzhen Shen, and Jiaxuan You.
\newblock Graphrouter: A graph-based router for llm selections.
\newblock In Y.~Yue, A.~Garg, N.~Peng, F.~Sha, and R.~Yu, editors, \emph{International Conference on Learning Representations}, volume 2025, pages 26186--26203, 2025.
\newblock URL \url{https://proceedings.iclr.cc/paper_files/paper/2025/file/41b6674c28a9b93ec8d22a53ca25bc3b-Paper-Conference.pdf}.

\bibitem[Fu et~al.(2025)Fu, Jiang, Chen, Fan, Geng, and Yang]{specensemble2025}
Jiale Fu, Yuchu Jiang, Junkai Chen, Jiaming Fan, Xin Geng, and Xu~Yang.
\newblock {Fast Large Language Model Collaborative Decoding via Speculation}.
\newblock In Aarti Singh, Maryam Fazel, Daniel Hsu, Simon Lacoste-Julien, Felix Berkenkamp, Tegan Maharaj, Kiri Wagstaff, and Jerry Zhu, editors, \emph{Proceedings of the 42nd International Conference on Machine Learning}, volume 267 of \emph{Proceedings of Machine Learning Research}, pages 17764--17782. PMLR, 13--19 Jul 2025.
\newblock URL \url{https://proceedings.mlr.press/v267/fu25c.html}.

\bibitem[Genest and Zidek(1986)]{genest1986combining}
Christian Genest and James~V. Zidek.
\newblock {Combining Probability Distributions: A Critique and an Annotated Bibliography}.
\newblock \emph{Statistical Science}, 1\penalty0 (1):\penalty0 114--135, 1986.
\newblock \doi{10.1214/ss/1177013825}.

\bibitem[{Granite Team, IBM}(2025)]{granite2024granite3}
{Granite Team, IBM}.
\newblock Granite-3.3-2b-instruct.
\newblock Hugging Face model card, 2025.
\newblock URL \url{https://huggingface.co/ibm-granite/granite-3.3-2b-instruct}.
\newblock Released April 16, 2025.

\bibitem[Guha et~al.(2024)Guha, Chen, Chow, Khare, and R\'{e}]{guha2024smoothie}
Neel Guha, Mayee~F. Chen, Trevor Chow, Ishan~S. Khare, and Christopher R\'{e}.
\newblock Smoothie: Label free language model routing.
\newblock In A.~Globerson, L.~Mackey, D.~Belgrave, A.~Fan, U.~Paquet, J.~Tomczak, and C.~Zhang, editors, \emph{Advances in Neural Information Processing Systems}, volume~37, pages 127645--127672. Curran Associates, Inc., 2024.
\newblock \doi{10.52202/079017-4053}.
\newblock URL \url{https://proceedings.neurips.cc/paper_files/paper/2024/file/e6b57a990462df5afa58d64ce2709db9-Paper-Conference.pdf}.

\bibitem[Guo(2025)]{lens2025}
Jizhou Guo.
\newblock Lens: Learning ensemble confidence from neural states for multi-llm answer integration, 2025.
\newblock URL \url{https://arxiv.org/abs/2507.23167}.

\bibitem[Hendrycks et~al.(2021)Hendrycks, Burns, Basart, Zou, Mazeika, Song, and Steinhardt]{hendrycks2021mmlu}
Dan Hendrycks, Collin Burns, Steven Basart, Andy Zou, Mantas Mazeika, Dawn Song, and Jacob Steinhardt.
\newblock {Measuring Massive Multitask Language Understanding}.
\newblock In \emph{International Conference on Learning Representations (ICLR)}, 2021.
\newblock URL \url{https://openreview.net/forum?id=d7KBjmI3GmQ}.

\bibitem[Hinton(2002)]{hinton2002products}
Geoffrey~E. Hinton.
\newblock {Training Products of Experts by Minimizing Contrastive Divergence}.
\newblock \emph{Neural Computation}, 14\penalty0 (8):\penalty0 1771--1800, 2002.
\newblock \doi{10.1162/089976602760128018}.

\bibitem[Hu et~al.(2025)Hu, Wang, Zhang, Zhou, Chen, Hu, Xiao, and Tan]{der2024}
Jinwu Hu, Yufeng Wang, Shuhai Zhang, Kai Zhou, Guohao Chen, Yu~Hu, Bin Xiao, and Mingkui Tan.
\newblock Efficient dynamic ensembling for multiple llm experts.
\newblock In James Kwok, editor, \emph{Proceedings of the Thirty-Fourth International Joint Conference on Artificial Intelligence, {IJCAI-25}}, pages 8095--8103. International Joint Conferences on Artificial Intelligence Organization, 8 2025.
\newblock \doi{10.24963/ijcai.2025/900}.
\newblock URL \url{https://doi.org/10.24963/ijcai.2025/900}.
\newblock Main Track.

\bibitem[Hu et~al.(2024)Hu, Bieker, Li, Jiang, Keigwin, Ranganath, Keutzer, and Upadhyay]{routerbench2024}
Qitian~Jason Hu, Jacob Bieker, Xiuyu Li, Nan Jiang, Benjamin Keigwin, Gaurav Ranganath, Kurt Keutzer, and Shriyash~Kaustubh Upadhyay.
\newblock Routerbench: A benchmark for multi-llm routing system, 2024.
\newblock URL \url{https://arxiv.org/abs/2403.12031}.

\bibitem[Huang et~al.(2024)Huang, Feng, Li, Xiang, Wang, Liu, and Qin]{huang2024deepen}
Yichong Huang, Xiaocheng Feng, Baohang Li, Yang Xiang, Hui Wang, Ting Liu, and Bing Qin.
\newblock Ensemble learning for heterogeneous large language models with deep parallel collaboration.
\newblock In A.~Globerson, L.~Mackey, D.~Belgrave, A.~Fan, U.~Paquet, J.~Tomczak, and C.~Zhang, editors, \emph{Advances in Neural Information Processing Systems}, volume~37, pages 119838--119860. Curran Associates, Inc., 2024.
\newblock \doi{10.52202/079017-3808}.
\newblock URL \url{https://proceedings.neurips.cc/paper_files/paper/2024/file/d8a6eb79f8ccaacbe7198a5caf3a0323-Paper-Conference.pdf}.

\bibitem[Hui et~al.(2024)Hui, Yang, Cui, Yang, Liu, Zhang, Liu, Zhang, Yu, Lu, Dang, Fan, Zhang, Yang, Men, Huang, Zheng, Miao, Quan, Feng, Ren, Ren, Zhou, and Lin]{hui2024qwencoder}
Binyuan Hui, Jian Yang, Zeyu Cui, Jiaxi Yang, Dayiheng Liu, Lei Zhang, Tianyu Liu, Jiajun Zhang, Bowen Yu, Keming Lu, Kai Dang, Yang Fan, Yichang Zhang, An~Yang, Rui Men, Fei Huang, Bo~Zheng, Yibo Miao, Shanghaoran Quan, Yunlong Feng, Xingzhang Ren, Xuancheng Ren, Jingren Zhou, and Junyang Lin.
\newblock Qwen2.5-coder technical report, 2024.
\newblock URL \url{https://arxiv.org/abs/2409.12186}.

\bibitem[Jacobs et~al.(1991)Jacobs, Jordan, Nowlan, and Hinton]{jacobs1991moe}
Robert~A. Jacobs, Michael~I. Jordan, Steven~J. Nowlan, and Geoffrey~E. Hinton.
\newblock {Adaptive Mixtures of Local Experts}.
\newblock \emph{Neural Computation}, 3\penalty0 (1):\penalty0 79--87, 1991.
\newblock \doi{10.1162/neco.1991.3.1.79}.

\bibitem[Jiang et~al.(2023)Jiang, Ren, and Lin]{jiang2023llmblender}
Dongfu Jiang, Xiang Ren, and Bill~Yuchen Lin.
\newblock {LLM}-blender: Ensembling large language models with pairwise ranking and generative fusion.
\newblock In Anna Rogers, Jordan Boyd-Graber, and Naoaki Okazaki, editors, \emph{Proceedings of the 61st Annual Meeting of the Association for Computational Linguistics (Volume 1: Long Papers)}, pages 14165--14178, Toronto, Canada, July 2023. Association for Computational Linguistics.
\newblock \doi{10.18653/v1/2023.acl-long.792}.
\newblock URL \url{https://aclanthology.org/2023.acl-long.792/}.

\bibitem[Li et~al.(2024)Li, Zhang, Yu, Fu, and Ye]{li2024moreagents}
Junyou Li, Qin Zhang, Yangbin Yu, Qiang Fu, and Deheng Ye.
\newblock More agents is all you need.
\newblock \emph{Transactions on Machine Learning Research}, 2024.
\newblock ISSN 2835-8856.
\newblock URL \url{https://openreview.net/forum?id=bgzUSZ8aeg}.

\bibitem[Li(2025)]{knnrouter2025}
Yang Li.
\newblock {Rethinking Predictive Modeling for LLM Routing: When Simple kNN Beats Complex Learned Routers}, 2025.
\newblock URL \url{https://doi.org/10.48550/arXiv.2505.12601}.
\newblock arXiv:2505.12601.

\bibitem[Liu et~al.(2025)Liu, Quan, Pan, Wu, Chen, and Lin]{coolfusion2024}
Cong Liu, Xiaojun Quan, Yan Pan, Weigang Wu, Xu~Chen, and Liang Lin.
\newblock Cool-fusion: Fuse large language models without training.
\newblock In Wanxiang Che, Joyce Nabende, Ekaterina Shutova, and Mohammad~Taher Pilehvar, editors, \emph{Proceedings of the 63rd Annual Meeting of the Association for Computational Linguistics (Volume 1: Long Papers)}, pages 10617--10627, Vienna, Austria, July 2025. Association for Computational Linguistics.
\newblock \doi{10.18653/v1/2025.acl-long.521}.
\newblock URL \url{https://aclanthology.org/2025.acl-long.521/}.

\bibitem[Lu et~al.(2024)Lu, Yuan, Lin, Lin, Yuan, Zhou, and Zhou]{lu2023zooter}
Keming Lu, Hongyi Yuan, Runji Lin, Junyang Lin, Zheng Yuan, Chang Zhou, and Jingren Zhou.
\newblock Routing to the expert: Efficient reward-guided ensemble of large language models.
\newblock In Kevin Duh, Helena Gomez, and Steven Bethard, editors, \emph{Proceedings of the 2024 Conference of the North American Chapter of the Association for Computational Linguistics: Human Language Technologies (Volume 1: Long Papers)}, pages 1964--1974, Mexico City, Mexico, June 2024. Association for Computational Linguistics.
\newblock \doi{10.18653/v1/2024.naacl-long.109}.
\newblock URL \url{https://aclanthology.org/2024.naacl-long.109/}.

\bibitem[Martins et~al.(2025)Martins, Fernandes, Alves, Guerreiro, Rei, Alves, Pombal, Farajian, Faysse, Klimaszewski, Colombo, Haddow, {de Souza}, Birch, and Martins]{martins2024eurollm}
Pedro~Henrique Martins, Patrick Fernandes, João Alves, Nuno~M. Guerreiro, Ricardo Rei, Duarte~M. Alves, José Pombal, Amin Farajian, Manuel Faysse, Mateusz Klimaszewski, Pierre Colombo, Barry Haddow, José~G.C. {de Souza}, Alexandra Birch, and André~F.T. Martins.
\newblock Eurollm: Multilingual language models for europe.
\newblock \emph{Procedia Computer Science}, 255:\penalty0 53--62, 2025.
\newblock ISSN 1877-0509.
\newblock \doi{https://doi.org/10.1016/j.procs.2025.02.260}.
\newblock URL \url{https://www.sciencedirect.com/science/article/pii/S1877050925006210}.
\newblock Proceedings of the Second EuroHPC user day.

\bibitem[Mavromatis et~al.(2024)Mavromatis, Karypis, and Karypis]{mavromatis2024packllm}
Costas Mavromatis, Petros Karypis, and George Karypis.
\newblock Pack of {LLM}s: Model fusion at test-time via perplexity optimization.
\newblock In \emph{First Conference on Language Modeling}, 2024.
\newblock URL \url{https://openreview.net/forum?id=5Nsl0nlStc}.

\bibitem[Nadeau and Bengio(2003)]{nadeau2003inference}
Claude Nadeau and Yoshua Bengio.
\newblock {Inference for the Generalization Error}.
\newblock \emph{Machine Learning}, 52:\penalty0 239--281, 2003.
\newblock \doi{https://doi.org/10.1023/A:1024068626366}.

\bibitem[Nguyen et~al.(2025)Nguyen, Dao, Hoang, Decugis, Manchanda, Chawla, and Doan]{metallm2024}
Quang~H. Nguyen, Thinh Dao, Duy~C. Hoang, Juliette Decugis, Saurav Manchanda, Nitesh~V. Chawla, and Khoa~D. Doan.
\newblock Metallm: A high-performant and cost-efficient dynamic framework for wrapping llms, 2025.
\newblock URL \url{https://arxiv.org/abs/2407.10834}.

\bibitem[Nitzan and Paroush(1982)]{nitzan1982optimal}
Shmuel Nitzan and Jacob Paroush.
\newblock {Optimal Decision Rules in Uncertain Dichotomous Choice Situations}.
\newblock \emph{International Economic Review}, 23\penalty0 (2):\penalty0 289--297, 1982.
\newblock \doi{10.2307/2526438}.

\bibitem[Ong et~al.(2025)Ong, Almahairi, Wu, Chiang, Wu, Gonzalez, Kadous, and Stoica]{ong2024routellm}
Isaac Ong, Amjad Almahairi, Vincent Wu, Wei-Lin Chiang, Tianhao Wu, Joseph~E. Gonzalez, M.~Waleed Kadous, and Ion Stoica.
\newblock Routellm: Learning to route llms with preference data.
\newblock In Y.~Yue, A.~Garg, N.~Peng, F.~Sha, and R.~Yu, editors, \emph{International Conference on Learning Representations}, volume 2025, pages 34433--34448, 2025.
\newblock URL \url{https://proceedings.iclr.cc/paper_files/paper/2025/file/5503a7c69d48a2f86fc00b3dc09de686-Paper-Conference.pdf}.

\bibitem[Park et~al.(2025)Park, Liu, Gong, and Choi]{lemcts2024}
Sungjin Park, Xiao Liu, Yeyun Gong, and Edward Choi.
\newblock Ensembling large language models with process reward-guided tree search for better complex reasoning.
\newblock In Luis Chiruzzo, Alan Ritter, and Lu~Wang, editors, \emph{Proceedings of the 2025 Conference of the Nations of the Americas Chapter of the Association for Computational Linguistics: Human Language Technologies (Volume 1: Long Papers)}, pages 10256--10277, Albuquerque, New Mexico, April 2025. Association for Computational Linguistics.
\newblock \doi{10.18653/v1/2025.naacl-long.515}.
\newblock URL \url{https://aclanthology.org/2025.naacl-long.515/}.

\bibitem[Prelec et~al.(2017)Prelec, Seung, and McCoy]{prelec2017surprising}
Dra{\v{z}}en Prelec, H.~Sebastian Seung, and John McCoy.
\newblock {A Solution to the Single-Question Crowd Wisdom Problem}.
\newblock \emph{Nature}, 541\penalty0 (7638):\penalty0 532--535, 2017.
\newblock \doi{10.1038/nature21054}.

\bibitem[{Qwen Team}(2025{\natexlab{a}})]{qwen2024qwen25}
{Qwen Team}.
\newblock {Qwen2.5 Technical Report}, 2025{\natexlab{a}}.
\newblock URL \url{https://arxiv.org/abs/2412.15115}.

\bibitem[{Qwen Team}(2025{\natexlab{b}})]{qwen2025qwen3}
{Qwen Team}.
\newblock {Qwen3 Technical Report}, 2025{\natexlab{b}}.
\newblock URL \url{https://arxiv.org/abs/2505.09388}.

\bibitem[Rodrigues and Pereira(2018)]{rodrigues2018deep}
Filipe Rodrigues and Francisco~C. Pereira.
\newblock {Deep Learning from Crowds}.
\newblock In \emph{Proceedings of the AAAI Conference on Artificial Intelligence}, volume~32, pages 1611--1618, 2018.
\newblock \doi{https://doi.org/10.1609/aaai.v32i1.11506}.

\bibitem[Shazeer et~al.(2017)Shazeer, Mirhoseini, Maziarz, Davis, Le, Hinton, and Dean]{shazeer2017moe}
Noam Shazeer, Azalia Mirhoseini, Krzysztof Maziarz, Andy Davis, Quoc Le, Geoffrey Hinton, and Jeff Dean.
\newblock Outrageously large neural networks: The sparsely-gated mixture-of-experts layer.
\newblock In \emph{International Conference on Learning Representations (ICLR)}, 2017.
\newblock URL \url{https://openreview.net/forum?id=B1ckMDqlg}.

\bibitem[Si et~al.(2023)Si, Shi, Zhao, Zettlemoyer, and Boyd-Graber]{si2023more}
Chenglei Si, Weijia Shi, Chen Zhao, Luke Zettlemoyer, and Jordan Boyd-Graber.
\newblock Getting {M}o{RE} out of mixture of language model reasoning experts.
\newblock In Houda Bouamor, Juan Pino, and Kalika Bali, editors, \emph{Findings of the Association for Computational Linguistics: EMNLP 2023}, pages 8234--8249, Singapore, December 2023. Association for Computational Linguistics.
\newblock \doi{10.18653/v1/2023.findings-emnlp.552}.
\newblock URL \url{https://aclanthology.org/2023.findings-emnlp.552/}.

\bibitem[Sikeridis et~al.(2025)Sikeridis, Ramdass, and Pareek]{pickllm2024}
Dimitrios Sikeridis, Dennis Ramdass, and Pranay Pareek.
\newblock Pickllm: Context-aware rl-assisted large language model routing.
\newblock In Qingyun Wang, Wenpeng Yin, Abhishek Aich, Yumin Suh, and Kuan-Chuan Peng, editors, \emph{AI for Research and Scalable, Efficient Systems}, pages 227--239. Springer Nature Singapore, 2025.
\newblock \doi{10.1007/978-981-96-8912-5_11}.

\bibitem[Stone(1961)]{stone1961opinion}
Mervyn Stone.
\newblock {The Opinion Pool}.
\newblock \emph{The Annals of Mathematical Statistics}, 32\penalty0 (4):\penalty0 1339--1342, 1961.
\newblock \doi{10.1214/aoms/1177704873}.

\bibitem[Sun et~al.(2025)Sun, Kagrecha, Manakul, Woodland, and Gales]{sun-etal-2025-skillaggregation}
Guangzhi Sun, Anmol Kagrecha, Potsawee Manakul, Phil Woodland, and Mark Gales.
\newblock {S}kill{A}ggregation: Reference-free {LLM}-dependent aggregation.
\newblock In Wanxiang Che, Joyce Nabende, Ekaterina Shutova, and Mohammad~Taher Pilehvar, editors, \emph{Proceedings of the 63rd Annual Meeting of the Association for Computational Linguistics (Volume 1: Long Papers)}, pages 15532--15548, Vienna, Austria, July 2025. Association for Computational Linguistics.
\newblock \doi{10.18653/v1/2025.acl-long.756}.
\newblock URL \url{https://aclanthology.org/2025.acl-long.756/}.

\bibitem[Tekin et~al.(2024)Tekin, Ilhan, Huang, Hu, and Liu]{tekin2024topla}
Selim~Furkan Tekin, Fatih Ilhan, Tiansheng Huang, Sihao Hu, and Ling Liu.
\newblock {LLM}-{TOPLA}: Efficient {LLM} ensemble by maximising diversity.
\newblock In Yaser Al-Onaizan, Mohit Bansal, and Yun-Nung Chen, editors, \emph{Findings of the Association for Computational Linguistics: EMNLP 2024}, pages 11951--11966, Miami, Florida, USA, November 2024. Association for Computational Linguistics.
\newblock \doi{10.18653/v1/2024.findings-emnlp.698}.
\newblock URL \url{https://aclanthology.org/2024.findings-emnlp.698/}.

\bibitem[Turkmen et~al.(2026)Turkmen, Buyukates, and Bastopcu]{infotheoretic2026}
Yigit Turkmen, Baturalp Buyukates, and Melih Bastopcu.
\newblock {Don't Always Pick the Highest-Performing Model: An Information Theoretic View of LLM Ensemble Selection}, 2026.
\newblock URL \url{https://doi.org/10.48550/arXiv.2602.08003}.
\newblock arXiv:2602.08003.

\bibitem[Walsh et~al.(2025)Walsh, Soldaini, Groeneveld, Lo, Arora, Bhagia, Gu, Huang, Jordan, Lambert, Schwenk, Tafjord, Anderson, Atkinson, Brahman, Clark, Dasigi, Dziri, Ettinger, Guerquin, Heineman, Ivison, Koh, Liu, Malik, Merrill, Miranda, Morrison, Murray, Nam, Poznanski, Pyatkin, Rangapur, Schmitz, Skjonsberg, Wadden, Wilhelm, Wilson, Zettlemoyer, Farhadi, Smith, and Hajishirzi]{olmo2025olmo2}
Evan~Pete Walsh, Luca Soldaini, Dirk Groeneveld, Kyle Lo, Shane Arora, Akshita Bhagia, Yuling Gu, Shengyi Huang, Matt Jordan, Nathan Lambert, Dustin Schwenk, Oyvind Tafjord, Taira Anderson, David Atkinson, Faeze Brahman, Christopher Clark, Pradeep Dasigi, Nouha Dziri, Allyson Ettinger, Michal Guerquin, David Heineman, Hamish Ivison, Pang~Wei Koh, Jiacheng Liu, Saumya Malik, William Merrill, Lester James~Validad Miranda, Jacob Morrison, Tyler Murray, Crystal Nam, Jake Poznanski, Valentina Pyatkin, Aman Rangapur, Michael Schmitz, Sam Skjonsberg, David Wadden, Christopher Wilhelm, Michael Wilson, Luke Zettlemoyer, Ali Farhadi, Noah~A. Smith, and Hannaneh Hajishirzi.
\newblock 2 {OLM}o 2 furious ({COLM}{\textquoteright}s version).
\newblock In \emph{Second Conference on Language Modeling}, 2025.
\newblock URL \url{https://openreview.net/forum?id=2ezugTT9kU}.

\bibitem[Wang et~al.(2023)Wang, Wei, Schuurmans, Le, Chi, Narang, Chowdhery, and Zhou]{wang2023selfconsistency}
Xuezhi Wang, Jason Wei, Dale Schuurmans, Quoc Le, Ed~Chi, Sharan Narang, Aakanksha Chowdhery, and Denny Zhou.
\newblock {Self-Consistency Improves Chain of Thought Reasoning in Language Models}.
\newblock In \emph{International Conference on Learning Representations (ICLR)}, 2023.
\newblock URL \url{https://openreview.net/forum?id=1PL1NIMMrw}.

\bibitem[Yang et~al.(2024)Yang, Zhang, Hui, Gao, Yu, Li, Liu, Tu, Zhou, Lin, Lu, Xue, Lin, Liu, Ren, and Zhang]{yang2024qwenmath}
An~Yang, Beichen Zhang, Binyuan Hui, Bofei Gao, Bowen Yu, Chengpeng Li, Dayiheng Liu, Jianhong Tu, Jingren Zhou, Junyang Lin, Keming Lu, Mingfeng Xue, Runji Lin, Tianyu Liu, Xingzhang Ren, and Zhenru Zhang.
\newblock Qwen2.5-math technical report: Toward mathematical expert model via self-improvement, 2024.
\newblock URL \url{https://arxiv.org/abs/2409.12122}.

\bibitem[Yu et~al.(2024)Yu, Kuo, Ye, Chang, and Li]{yu2024gac}
Yao-Ching Yu, Chun-Chih Kuo, Ziqi Ye, Yu-Cheng Chang, and Yueh-Se Li.
\newblock Breaking the ceiling of the {LLM} community by treating token generation as a classification for ensembling.
\newblock In Yaser Al-Onaizan, Mohit Bansal, and Yun-Nung Chen, editors, \emph{Findings of the Association for Computational Linguistics: EMNLP 2024}, pages 1826--1839. Association for Computational Linguistics, November 2024.
\newblock \doi{10.18653/v1/2024.findings-emnlp.99}.
\newblock URL \url{https://aclanthology.org/2024.findings-emnlp.99/}.

\bibitem[Zellers et~al.(2019)Zellers, Holtzman, Bisk, Farhadi, and Choi]{zellers2019hellaswag}
Rowan Zellers, Ari Holtzman, Yonatan Bisk, Ali Farhadi, and Yejin Choi.
\newblock {H}ella{S}wag: Can a machine really finish your sentence?
\newblock In Anna Korhonen, David Traum, and Llu{\'i}s M{\`a}rquez, editors, \emph{Proceedings of the 57th Annual Meeting of the Association for Computational Linguistics}, pages 4791--4800, Florence, Italy, July 2019. Association for Computational Linguistics.
\newblock \doi{10.18653/v1/P19-1472}.
\newblock URL \url{https://aclanthology.org/P19-1472/}.

\bibitem[Zhang et~al.(2026)Zhang, Li, Wang, Chen, Zhang, Ye, Feng, Wang, Xu, Bai, and Hu]{avengers2025}
Yiqun Zhang, Hao Li, Chenxu Wang, Linyao Chen, Qiaosheng Zhang, Peng Ye, Shi Feng, Xinrun Wang, Jia Xu, Lei Bai, and Shuyue Hu.
\newblock {The Avengers: A Routing Recipe for Collective Intelligence in Language Models}.
\newblock In \emph{Proceedings of the AAAI Conference on Artificial Intelligence}, pages 34870–--34878, 2026.
\newblock \doi{10.1609/aaai.v40i41.40790}.

\end{thebibliography}
\endgroup

\clearpage
\appendix

%\section*{Appendices}

\section{Reproducibility and experimental details}
\label[appendix]{app:setup}

\paragraph{Reproducibility.}
The implementation, cached expert answers, and a \texttt{Makefile} are available
at \url{https://github.com/distributedcognition/jury}. \texttt{make} rebuilds
all tables and experimental figures from the cache in about ten minutes on a laptop without
rerunning any language model, while \texttt{make check} verifies the
regenerated results against the manuscript. The more expensive refitting,
sub-pool, and SkillAggregation analyses include their measurements and
dedicated rerun scripts; full commands and runtime details are provided in the
repository README.
% ===========================================================================
%\section{Experimental details}
%\label[appendix]{app:setup}
% ===========================================================================

\paragraph{Experts.}
The pool includes: Qwen2.5 at 0.5B/1.5B/3B \cite{qwen2024qwen25}, Qwen3-0.6B \cite{qwen2025qwen3}, Qwen2.5-Math-1.5B \cite{yang2024qwenmath}, Qwen2.5-Coder-1.5B \cite{hui2024qwencoder}, OLMo-2-1B \cite{olmo2025olmo2}, Granite-3.3-2B \cite{granite2024granite3} and EuroLLM-1.7B \cite{martins2024eurollm}. Each expert is run under its own chat template. Qwen3 is run with thinking mode disabled, because our single-pass multiple-choice protocol and the 256-token generation budget cannot accommodate it. \Cref{tab:experts} reports each expert's overall and per-source accuracy on all $1560$ benchmark questions, together with the source-only routing baseline and the per-question oracle.

\begin{table}[t]
\centering
\begin{tabular}{lrrrrrr}
\toprule
Expert & Params (B) & All & MMLU & GSM8K & ARC-C & HellaSwag \\
\midrule
\texttt{qwen3b} & 3.00 & \textbf{75.0} & 71.2 & 57.5 & 80.8 & 74.2 \\
\texttt{granite2b} & 2.50 & 66.6 & 64.6 & 69.2 & 68.0 & 65.5 \\
\texttt{qwen1.5b} & 1.50 & 66.5 & 67.1 & 40.0 & 75.2 & 63.0 \\
\texttt{qwen-coder1.5b} & 1.50 & 52.3 & 56.2 & 26.7 & 61.3 & 46.8 \\
\texttt{qwen3-0.6b} & 0.75 & 45.4 & 45.4 & 53.3 & 48.5 & 40.7 \\
\texttt{olmo2-1b} & 1.50 & 41.0 & 47.1 & 74.2 & 42.0 & 31.0 \\
\texttt{qwen0.5b} & 0.50 & 32.4 & 37.5 & 39.2 & 32.8 & 28.5 \\
\texttt{qwen-math} & 1.50 & 28.9 & 28.7 & 52.5 & 26.2 & 27.0 \\
\texttt{eurollm1.7b} & 1.70 & 24.4 & 22.9 & 10.8 & 25.7 & 26.5 \\
\midrule
Random expert & 1.61 & 48.1 & 49.0 & 47.0 & 51.2 & 44.8 \\
Best expert per source & -- & 76.3 & -- & -- & -- & -- \\
Oracle (best per question) & -- & 94.2 & 94.2 & 89.2 & 96.2 & 93.3 \\
\bottomrule
\end{tabular}

\caption{The nine experts alone, over all $1560$ questions --- which is why the
per-source figures here differ by a point or two from \Cref{tab:bysource}, where
every rule including this one is scored on held-out halves averaged over the ten
splits. ``Best expert per source''
is a source-only routing baseline, not a bound on query-conditioned routing; the
oracle is what a perfect per-question selector would score.}
\label{tab:experts}
\end{table}

\begin{figure}[t]
\centering
\includegraphics[width=0.9\textwidth]{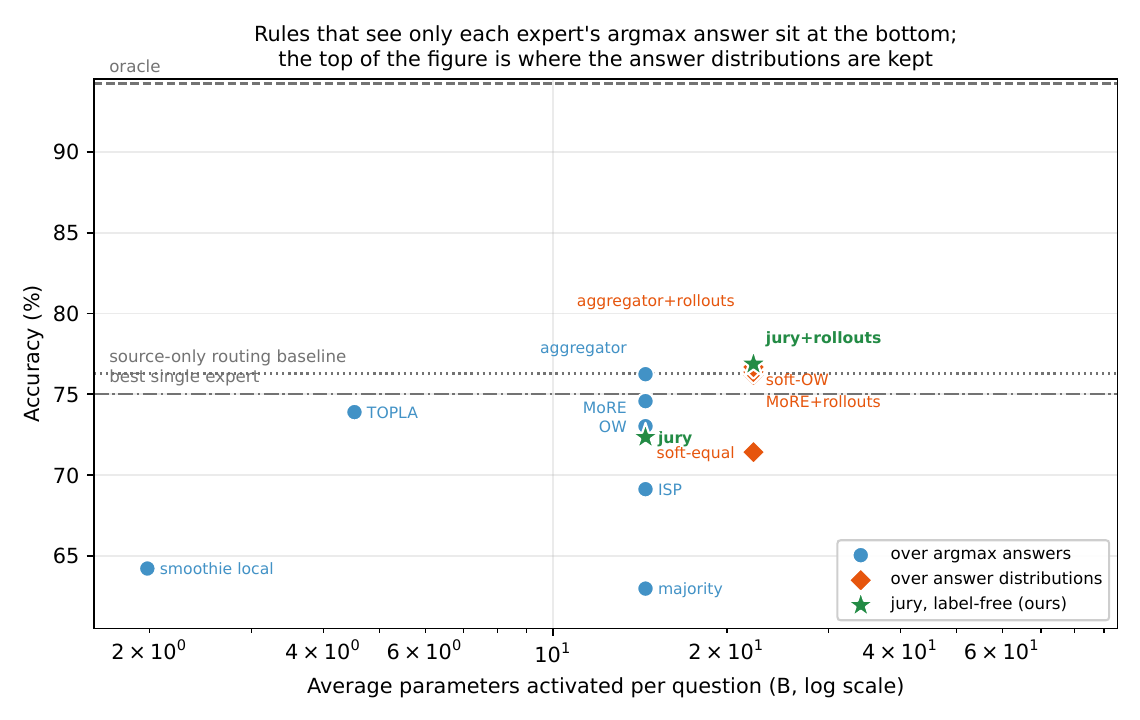}
\caption{Accuracy against average parameters activated per question. The methods
that consume answer distributions occupy the top of the figure; among them the
jury and the supervised linear fusion are separated by about half a point.}
\label{fig:frontier}
\end{figure}

\paragraph{Choice of sources.}
The six MMLU subjects are world religions, professional law, high-school
biology, prehistory, marketing and high-school geography, with $40$ questions
drawn uniformly without replacement from the \texttt{test} split of each.
ARC-Challenge and HellaSwag are the multiple-choice sets that RouterBench
\cite{routerbench2024} scores $0/1$, and a binary score is what this paper's
framework requires end to end. The emission of \eqref{eq:emission} states whether
an expert produced the true answer, the strata of \Cref{app:strata} are defined
by whether the gold answer is on the table, and every interval is a proportion of
questions answered correctly. RouterBench's other sets are graded (GSM8K on five
levels, MT-Bench on ten), and a fractional score identifies no gold answer.
ARC and HellaSwag also have the four-option format the prompt and scorer
already assume. Items with three or five options are dropped rather than
reshaped, since $m = |\cand| - 1$ is what the error channel divides by. Sampling
uses a fixed seed, and the question identifiers and labels are committed with
the repository.

\paragraph{Effect of the added sources on the pool.}
Both added sets favour the largest generalist and penalise the specialists.
\texttt{qwen3b} rises from $66.7\%$ on the original $360$ questions to $75.0\%$,
while \texttt{olmo2-1b}, the pool's best expert on GSM8K, falls from $56.1\%$ to
$41.0\%$. The gap between the best and second-best expert widens from $0.6$ to
$8.4$ points, so the pool measured here is a harder case for combination than
the original one. The source-only routing baseline ($76.3\%$) lies $1.3$ points above the
best single expert ($75.0\%$).

\paragraph{Answer protocol.}
Multiple choice is scored under a prompt that states the question and the four
options and asks for the letter alone. The argmax over the
\texttt{A}/\texttt{B}/\texttt{C}/\texttt{D} continuation tokens is the answer,
and the normalised probability of the chosen letter is the expert's
self-confidence. GSM8K generation is capped at $256$ new tokens. The final
number is taken from a match of \texttt{answer is}~$\ldots$ if present, and
otherwise it is the last number in the completion. A completion containing no
number is recorded as a non-answer and contributes nothing to any combiner.
Numeric answers are compared after stripping separators and rounding at
$10^{-4}$. On GSM8K the eight samples are eight real generations; on multiple
choice they are a quantised copy of the distribution that the single forward
pass already produces.

\paragraph{Candidate sets.}
On multiple choice, $|\cand(q)|$ averages $2.98$ with greedy answers and equals
$4$ after sampling on all but four of the $1440$ questions. On GSM8K it averages
$4.72$ greedily and $17.84$ after sampling, with a maximum of $38$, counting
the greedy and the sampled answers together. Thus the
uniform-error channel nearly coincides with the task answer space on sampled
multiple-choice questions, whereas on GSM8K it is defined by the pool's observed
behaviour.

\subsection{The supervised linear fusion}
\label[appendix]{app:supervised}
``Weights fitted on labels'' can mean two things: estimating each expert's
reliability from the labels and converting it into a weight, or learning
directly the weights that maximise the combined prediction's accuracy. The
baseline we call the supervised linear fusion is the first. Its reliability for
expert $e$ is the accuracy of $e$'s greedy answer on the training half,
\begin{equation}
  \hat{\pi}^{\mathrm{sup}}_e \;=\; \frac{1}{|Q_{\mathrm{tr}}|}
  \sum_{q \in Q_{\mathrm{tr}}} \mathbf{1}\{a_{e,q} = y_q\},
  \label{eq:sup}
\end{equation}
over every training question, unanimous ones included and a non-answer
counted as wrong, clipped to $[10^{-3}, 1 - 10^{-3}]$ and turned into the
log-odds weight $w_e = \log\frac{\hat{\pi}^{\mathrm{sup}}_e\, m}{1 -
\hat{\pi}^{\mathrm{sup}}_e}$ of \Cref{prop:pool}, with $m = |\cand(q)| - 1$ on
the same candidate set as the jury and applied to the same distributions
$p_{e,q}$. It is therefore the fusion rule of \Cref{prop:pool} with its
reliabilities read off labels, and the comparison with the jury is between
two ways of estimating the reliabilities of one rule, not against every way of
fitting weights from labels; the discriminative baselines of \Cref{tab:main}
(MoRE, LLM-TOPLA and our own answer scorer) are the supervised references of
the second kind. The two estimators also differ in three implementation
conventions, which we make explicit for completeness: the jury's M-step \eqref{eq:mstep-dist} averages the
posterior-weighted mass on each candidate over the questions the expert
answered, excludes unanimous questions (which constrain nothing under the
observed support) and clips to $[0.02, 0.98]$; the supervised estimate uses
the greedy answer's correctness, every training question and a wider clip.
A distributional supervised plug-in based on the mass assigned to the gold
answer gives the same qualitative conclusion conditional on the fitted models
(\Cref{app:plugin}).
Optimal Weight is the greedy-answer counterpart with the same reliabilities:
on multiple choice it uses the task-defined $K = 4$, whereas the jury uses
$m = |\cand(q)| - 1$. The two coincide whenever all four options are present
in $\cand(q)$, but can differ on greedy inputs when the pool proposes only two
or three options; see \Cref{sec:comparison}. We therefore also evaluate Optimal Weight refitted
under the jury's convention, which is then the same rule with $\pi$ read off
labels.

The greedy-answer comparison isolates the same distinction as \Cref{prop:reduce}: both methods belong to the same family of
log-odds weighted votes, but estimate the reliabilities differently. The jury
scores $72.3 \pm 1.1\%$ and Optimal Weight, with its task-defined $K = 4$, scores
$73.0 \pm 1.4\%$. Refitted under the jury's own convention, $m = |\cand(q)| - 1$,
so that the two rules differ in nothing but how $\pi$ is estimated, Optimal
Weight scores $73.3 \pm 1.1\%$, a difference of $-0.8$ points from the greedy
jury ($[-1.8, +0.1]$, $p = 0.10$) that is not significant. The signed regime of
\Cref{prop:pool} is not only theoretical: the weakest expert is fitted at $\pi = 0.23 < 0.25$ and therefore receives a negative weight.

\subsection{Reachability and question strata}
\label[appendix]{app:strata}

We consider four settings:
weighted voting over greedy answers, weighted voting over sampled modes,
non-negative fusion over answer distributions, and signed fusion over answer
distributions. Reachability is decided exactly by a linear program that maximises
the margin $\min_{u \ne v} \sum_e w_e (p_e(v) - p_e(u))$ of the gold answer $v$ over
every alternative; the gold answer is reachable when the optimum is positive. For
the two fusions we solve it with a normalisation that only fixes the scale of $w$:
over the simplex ($w_e \ge 0$, $\sum_e w_e = 1$), a linear opinion pool, and over
the box $w \in [-1, 1]^{|\experts|}$, the signed fusion of \Cref{prop:pool}.

Some weighted vote over greedy answers reaches the gold answer on $94.2\%$ of the
questions, yet no combiner in \Cref{tab:main} exceeds $77\%$, so expressiveness
is not the bottleneck. With signed weights the fusion reaches $99.7\%$,
essentially every gold answer that was sampled at all. With non-negative weights,
reachability rises only from $94.2\%$ to $95.8\%$. Yet \Cref{sec:comparison}
shows a much larger accuracy gain from using the distributions. Their main benefit
is therefore not simply to make new answers reachable, but to change how the
already reachable candidates are ranked.

Candidate availability sets an absolute ceiling on any combiner restricted to
$\cand(q)$. With greedy answers, no
expert proposes the gold answer on $5.8\%$ of the questions. Repeated sampling
restores most of these answers, leaving only $0.3\%$ of questions on which the gold
answer is never sampled. Each is an error for every rule confined to the
corresponding candidate set: the first for rules over greedy answers, the
second for rules over the sampled distributions.

\begin{table}[t]
\centering
\begin{tabular}{lrrrrr}
\toprule
Reachable set & All & MMLU & GSM8K & ARC-C & HellaSwag \\
\midrule
Weighted vote, greedy answers & 94.2 & 94.2 & 89.2 & 96.2 & 93.3 \\
Weighted vote, sampled modes & 95.6 & 95.0 & 92.5 & 97.0 & 95.2 \\
Distribution fusion, $\exists\,w \ge 0$ & 95.8 & 96.7 & 91.7 & 96.8 & 95.3 \\
Distribution fusion, $\exists\,w$ signed & 99.7 & 100.0 & 95.8 & 100.0 & 100.0 \\
Upper bound: gold sampled at all & 99.7 & 100.0 & 96.7 & 100.0 & 100.0 \\
\bottomrule
\end{tabular}

\caption{How often the gold answer is reachable under four evidence and weighting
regimes. For a weighted vote, the reachable set is exactly the set of answers
proposed by the experts; distributional reachability is decided exactly by a linear program over non-negative or signed
weights.}
\label{tab:reach}
\end{table}

\paragraph{Where methods separate.}
The \emph{plurality} and \emph{minority} strata defined in
\Cref{tab:main} contain $60.9\%$ and $33.3\%$ of the questions,
respectively. The remaining $5.8\%$ are \emph{unreachable} from
greedy answers: no expert proposes the gold answer. On minority
questions, an average of $2.4$ experts support the gold answer,
against $4.3$ for a plurality answer. Weighted voting can recover
the gold answer by assigning more weight to its supporters;
ties also allow majority voting to answer some minority questions
correctly ($6.2\%$ in \Cref{tab:main}). The jury remains $17.4$ percentage
points below the greedy-answer oracle ($94.3\%$). Most of this gap lies in the
minority stratum, where the jury recovers $55.5\%$ of the questions, compared
with $95.7\%$ in the plurality stratum.

% ===========================================================================
\section{Robustness of the comparison}
\label[appendix]{app:robustness}
% ===========================================================================

\subsection{Accuracy by source}
\label[appendix]{app:bysource}

\begin{table}[t]
\centering
\small
\begin{tabular}{lrrrrr}
\toprule
Method & MMLU & GSM8K & ARC-C & HellaSwag & Overall \\
\midrule
Best single expert & 72.6 & 59.8 & 80.2 & 73.7 & 75.0 \\
\quad + self-consistency on GSM8K & 72.6 & 64.7 & 80.2 & 73.7 & 75.3 \\
Majority vote & 65.2 & 71.5 & 67.5 & 55.8 & 63.0 \\
Equal-weight pool & 68.8 & 78.3 & 75.4 & 67.1 & 71.4 \\
MoRE + repeated sampling & 72.5 & 79.8 & 80.2 & 73.4 & 76.4 \\
Supervised linear fusion & 74.3 & 80.2 & 79.7 & 72.7 & 76.2 \\
\textbf{Jury (ours) + repeated sampling} & 74.8 & 79.8 & 80.1 & 73.9 & 76.9 \\
\midrule
Questions & 240 & 120 & 600 & 600 & 1560 \\
\bottomrule
\end{tabular}

\caption{Accuracy by source dataset. Three quarters of the benchmark is ARC and
HellaSwag, and they reorder the pool, so an overall number alone cannot say whether
a result holds across the four sources or rests on one. Means over the same ten
splits.}
\label{tab:bysource}
\end{table}

\Cref{tab:bysource} breaks the margin over the best expert down by source. The
jury is ahead by $+20.0$ points on GSM8K and $+2.2$ on MMLU, and within a few
tenths of a point on ARC-Challenge ($-0.1$) and HellaSwag ($+0.2$). Weighted by
size, GSM8K contributes $+1.54$ of the $+1.92$, so most of the margin comes from
the source where \texttt{qwen3b} is weakest ($57.5\%$ over all questions,
\Cref{tab:experts}) and a different expert, \texttt{olmo2-1b}, is the best in
the pool ($74.2\%$). Here, the other experts provide information that
\texttt{qwen3b} lacks. Eight-sample
self-consistency on GSM8K, from the same cached samples, lifts \texttt{qwen3b}
there by $+4.9$ points and leaves the other sources unchanged; the jury is still
$+15.1$ ahead of it on GSM8K and $+1.6$ overall. On the two sets
where \texttt{qwen3b} is also the pool's best expert, combining does not improve
on it, but it incurs essentially no accuracy loss either. A deployment cannot know in advance which
questions are which. The margin is smaller than the $+11$ points of the original
$360$-question pool because the added sources widened the lead of the best
expert (\Cref{app:setup}); what a combiner can gain is bounded by how much the
pool's members disagree usefully.

\subsection{Refitting and resampling}
\label[appendix]{app:refit}

\begin{table}[t]
\centering
\footnotesize
\resizebox{\linewidth}{!}{\begin{tabular}{lrrr}
\toprule
& \multicolumn{3}{c}{$\Delta$ vs.\ jury, $95\%$ interval} \\
\cmidrule(lr){2-4}
Resampled & Questions & Split and fit & Both \\
\midrule
Supervised linear fusion & +0.6~[+0.1, +1.2] & +0.7~[+0.0, +1.4] & +0.6~[-0.6, +2.0] \\
MoRE + repeated sampling & +0.5~[-0.4, +1.5] & +0.5~[-0.6, +1.7] & +0.5~[-1.6, +2.6] \\
Answer scoring + repeated sampling & +0.1~[-1.0, +1.2] & +0.2~[-1.3, +1.7] & +1.9~[-1.0, +5.1] \\
Equal-weight pool & +5.5~[+4.1, +7.0] & +5.6~[+4.1, +7.2] & +5.5~[+2.9, +8.3] \\
Majority vote & +14.1~[+12.0, +16.3] & +14.2~[+12.2, +16.3] & +14.1~[+10.5, +17.8] \\
\bottomrule
\end{tabular}
}
\caption{The same paired differences, resampling one more thing at a time.
\emph{Questions} is the paired bootstrap of \Cref{tab:main}, which holds the ten
fitted models fixed. \emph{Split and fit} draws a fresh stratified split and
refits every method on it. \emph{Both} additionally resamples the $1560$ questions
with replacement. The first column is conditional on the ten fitted models; the
other two include the variability of refitting, and their centres are means
over the replicas rather than the point estimates of the first column. $2000$
replicas each.}
\label{tab:refit}
\end{table}

The paired bootstrap resamples questions while holding fixed the ten fitted
models, so it is an interval for the difference \emph{given those fits}.
\Cref{tab:refit} adds one source of variability at a time, over $2000$ replicas
each. Redrawing the split and refitting every method (middle column) barely
changes anything. The majority-vote row goes from $[+12.0, +16.3]$ to
$[+12.2, +16.3]$ and the equal-weight row from $[+4.1, +7.0]$ to $[+4.1, +7.2]$.
The three comparisons against fitted rules widen by a fifth to a third, against
about $70\%$ on the original $360$-question benchmark, because each method is
now fitted on $780$ training questions rather than $180$.

The right column is the nested bootstrap. The questions are resampled with
replacement, then split, and everything is refitted inside the replica, with
every copy of a question assigned to the same side of the split. Only this
column includes, besides the sampling of questions, the variability of
refitting every method on a fresh training half.
Intervals widen by $70\%$ for majority voting, by about $140\%$ for the
comparison with the supervised fusion, and by about $180\%$ for the answer
scorer, whose fit varies most from replica to replica. The margin over the equal-weight pool ($+5.5$, $[+2.9, +8.3]$) and
over majority voting ($+14.1$, $[+10.5, +17.8]$) both survive. In this column the intervals against the supervised rules
include zero: the jury's $+0.6$ over the supervised linear fusion, significant
given the ten fits, has interval $[-0.6, +2.0]$ once refitting is resampled
as well. A bootstrap training half holds about $63\%$ distinct questions, so
every fitted method is fitted on less data than the protocol gives it; the
centre reported is the mean replica difference, which is why the answer
scorer's moves from $+0.1$ to $+1.9$: its fit loses more than the jury's on the
smaller effective training set.

\paragraph{Equivalence with the supervised linear fusion.}
An interval that contains zero says only that no difference was detected.
Equivalence is a separate claim, which requires a margin of practically
negligible difference and an interval that lies inside it. The $90\%$
interval for the jury minus the supervised linear fusion is $[+0.2, +1.1]$,
conditional on the ten fits: the jury is ahead at the $5\%$ level one-sided,
and two one-sided tests support equivalence at a margin of $\pm 1.1$ points
or wider, against $\pm 2.3$ on the original benchmark. The nested $90\%$
interval, $[-0.4, +1.8]$, neither separates the two nor lies within $\pm 1$
point, so no equivalence claim is available at that margin once refitting is
resampled. The results therefore show no significant supervised advantage on
this benchmark, and a label-free advantage that is significant conditional on
the fits but not once the variability of refitting is included.

\paragraph{The distributional supervised plug-in.}
\label[appendix]{app:plugin}
The supervised linear fusion of \eqref{eq:sup} estimates each reliability from
the greedy answer's correctness. A distributional supervised plug-in uses the
mass assigned to the gold answer in place of the posterior-weighted mass in
\eqref{eq:mstep-dist}, while retaining the supervised baseline's averaging and
clipping conventions:
\begin{equation}
  \hat{\pi}^{\mathrm{sup,dist}}_e \;=\; \frac{1}{|Q_{\mathrm{tr}}|}
  \sum_{q \in Q_{\mathrm{tr}}} p_{e,q}(y_q).
  \label{eq:sup-dist}
\end{equation}
The average includes every training question. We define $p_{e,q}(y_q) = 0$
when the gold answer lies outside $\cand(q)$; it is also zero whenever the
expert assigns no mass to an in-support gold answer. We apply the clipping
interval of \eqref{eq:sup}, $[10^{-3}, 1 - 10^{-3}]$, and convert the
resulting reliabilities into the log-odds weights of \Cref{prop:pool}. On the
same ten splits it reaches $76.6 \pm 1.4\%$, between the
greedy-correctness estimate ($76.2 \pm 1.4\%$) and the jury
($76.9 \pm 1.3\%$); the jury's paired difference over it is $+0.3$
($[+0.1, +0.6]$, $p = 0.02$) conditional on the ten fits. We did not run the
nested bootstrap for this variant. The conclusion is the one of
\Cref{sec:comparison}: reading the reliabilities off labels, by either
estimator, buys no measurable accuracy over estimating them from agreement on
this benchmark.

\subsection{Per-task reliabilities}
\label[appendix]{app:spec}

\begin{table}[t]
\centering
\begin{tabular}{llrrr}
\toprule
Weights & Labels & Global & Per task & $\Delta$ \\
\midrule
Supervised linear fusion & yes & 76.2 $\pm$ 1.4 & 76.4 $\pm$ 1.3 & +0.2 \\
Jury (ours) & no & 76.9 $\pm$ 1.3 & 77.3 $\pm$ 1.3 & +0.4 \\
\midrule
\multicolumn{5}{l}{\emph{Label-free minus supervised, same parameterisation}} \\
\quad global & & +0.6 [+0.1, +1.2] & & \\
\quad per task & & +0.9 [+0.3, +1.5] & & \\
\bottomrule
\end{tabular}

\caption{One reliability per expert against one per expert and task, on both sides
of the supervision line. The lower block is the label-free minus supervised
difference at each parameterisation, with its $95\%$ paired bootstrap interval
over questions.}
\label{tab:spec}
\end{table}

A single reliability per expert is a strong assumption for this pool:
\texttt{qwen3b} scores $76.5\%$ on multiple choice against $57.5\%$ on
arithmetic, and \texttt{olmo2-1b} the reverse ($38.3\%$ against $74.2\%$).
\Cref{tab:spec} relaxes this restriction on both sides. The supervised linear
fusion fits its weights on each task's training questions, and the jury runs EM
separately on each task's configurations. The task is part of the query, so
every rule may read it. Both sides improve, and the jury stays ahead: $77.3 \pm 1.3\%$
against $76.4 \pm 1.3\%$, a difference of $+0.9$ ($[+0.3, +1.5]$).
The label-free side improves more ($+0.4$
against $+0.2$), and the $90\%$ interval for label-free minus supervised moves
from $[+0.2, +1.1]$ to $[+0.4, +1.4]$. Whether supervision wins with a richer
parameterisation, per subject or conditioned on the query, remains an open
question.

\subsection{SkillAggregation, OW-L and OW-I on the multiple-choice questions}
\label[appendix]{app:skillagg}

\begin{table}[t]
\centering
\footnotesize
\resizebox{\linewidth}{!}{\begin{tabular}{lllrr}
\toprule
Method & Labels & Input & Accuracy & $\Delta$ vs.\ jury~[95\% CI] \\
\midrule
Majority vote & no & greedy & 62.3 $\pm$ 1.8 & +14.6$^{*}$~[+12.4, +16.9] \\
Inverse Surprising Popularity (ISP) & no & greedy & 68.9 $\pm$ 1.7 & +7.9$^{*}$~[+6.1, +9.7] \\
Optimal Weight, accuracies learned (OW-L) & no & greedy & 71.3 $\pm$ 1.7 & +5.3$^{*}$~[+3.8, +6.9] \\
Optimal Weight, ISP pseudo-labels (OW-I) & no & greedy & 71.6 $\pm$ 1.6 & +5.2$^{*}$~[+3.6, +6.7] \\
Equal-weight pool & no & distributions & 70.8 $\pm$ 1.9 & +5.9$^{*}$~[+4.4, +7.5] \\
\midrule
SkillAggregation~\cite{sun-etal-2025-skillaggregation} & no & greedy & 66.8 $\pm$ 1.8 & +10.2$^{*}$~[+8.2, +12.2] \\
\quad + answer distributions & no & distributions & 73.2 $\pm$ 1.9 & +3.5$^{*}$~[+2.2, +4.8] \\
\quad\quad + development set & 250 & distributions & 72.9 $\pm$ 1.7 & +3.7$^{*}$~[+2.4, +5.0] \\
\quad\quad + best observed epoch and $\lambda$ (test oracle) & test & distributions & 73.5 $\pm$ 1.9 & +3.1 \\
SkillAggregation-X~\cite{sun-etal-2025-skillaggregation} & no & distributions & 49.6 $\pm$ 1.8 & +26.6$^{*}$~[+24.0, +29.3] \\
\quad + development set & 250 & distributions & 73.2 $\pm$ 1.3 & +3.4$^{*}$~[+2.1, +4.8] \\
\quad + best observed epoch and $\lambda$ (test oracle) & test & distributions & 74.0 $\pm$ 1.3 & +2.7 \\
\midrule
Supervised linear fusion & yes & distributions & 75.9 $\pm$ 1.4 & +0.7$^{*}$~[+0.2, +1.3] \\
\textbf{Jury (ours), greedy answers} & \textbf{no} & greedy & 72.4 $\pm$ 1.1 & +4.3$^{*}$~[+3.0, +5.7] \\
\textbf{\quad + repeated sampling} & \textbf{no} & distributions & 76.6 $\pm$ 1.4 &  \\
\bottomrule
\end{tabular}}
\caption{The $1440$ multiple-choice questions only, over the same ten splits.
Every rule other than SkillAggregation, OW-L and OW-I is the one of \Cref{tab:main}, fitted as
there and scored on the multiple-choice test questions; OW-L and OW-I are fitted,
without labels, on the training half of each split. \emph{Labels} is what the
rule reads: SkillAggregation's development set is $250$ labelled questions from the
training half, used to choose $\lambda$ and the epoch, and \emph{test} marks the
best observed epoch and $\lambda$ over the $40$ epochs, selected on the test
questions themselves. $\Delta$ is the paired difference against the jury with
repeated sampling, with a $95\%$ bootstrap interval over questions; $^{*}$ marks
$p < 0.05$. The test-oracle rows give the point difference alone: the selection
used the same questions, so no interval or test is reported.}
\label{tab:skillagg}
\end{table}

SkillAggregation \cite{sun-etal-2025-skillaggregation} assumes a fixed label
space in its published multi-class formulation, so we compare against it on the
multiple-choice questions and leave GSM8K out. The released code is binary. We implement the multi-class extension of the authors'
Appendix~F: a context encoder whose bottleneck predicts the class, one confusion
matrix per judge, a regulariser on those matrices, and the
conditional-independence posterior as the group estimate. The encoder and the
training recipe follow the release (GPT-2 trained end to end, batch $8$, learning
rate $10^{-4}$, every judge initialised at skill $0.7$, $\lambda = 0.1$), with the
question and its four options as the context. We make two changes. We train for
$40$ epochs where the release uses $10$; this horizon was chosen after
inspecting the accuracy trajectory against the gold answers, so the
no-development-set configuration should not be read as free of benchmark-level
protocol tuning. For a judge that reports a
distribution, the posterior uses that distribution as a fractional count, which
reduces to the authors' posterior for a single vote.

We follow their transductive fitting protocol, which is more generous than the
jury's. The model is fitted on the experts' answers to all $1440$ questions, test
questions included, without reading a label, whereas the jury is fitted on the
training half only. For the distributional rows, SkillAggregation receives the
same per-expert answer distributions used by the jury and the equal-weight pool,
from the same eight samples per expert and question. The authors also choose
$\lambda$ and the checkpoint on $250$ labelled questions; we report the method
without that development set (the released $\lambda$, the last epoch), with it,
and at the best observed epoch and $\lambda$ over the $40$ epochs, selected on
the test questions themselves, one choice for all ten splits. That last row is a
test oracle: no protocol could deliver it, and because the selection used the
test questions we report its difference from the jury without an interval.

\paragraph{OW-L and OW-I.}
\citet{beyondmajority2025} also give two label-free estimates of the accuracies
that Optimal Weight needs, and they are the closest published counterparts of
the jury's M-step. OW-L fits the accuracies $x_1, \dots, x_N$ by least squares to
the empirical second-order table $\hat P(A_i = s_k \mid A_j = s_l)$, using the
closed form that table takes under their model, in which every wrong answer is
equally likely (their Equation~6 and Appendix~F.2). OW-I runs ISP and takes each
expert's rate of agreement with ISP's verdicts as its accuracy. Both then weight
expert $i$ by $\log[(K-1)\,x_i/(1 - x_i)]$ with the task-defined $K = 4$ and
take the weighted plurality of the greedy answers. Like ISP they need a fixed
label space, so we evaluate them on the multiple-choice questions, with the
agreement table and the pseudo-labels estimated on the unlabelled training half
of each split, as for ISP and the jury; fitting them on all $1440$ questions
instead, as in their own protocol, moves either result by at most $0.2$ points.
We solve OW-L's least-squares problem with L-BFGS-B from five starts and keep
the lowest loss.

\Cref{tab:skillagg} gives the result. On greedy answers OW-L reaches $71.3\%$
and OW-I $71.6\%$, both above ISP ($68.9\%$) and majority voting ($62.3\%$), and
$1.0$ and $0.9$ points below the greedy jury ($72.4\%$), with paired intervals
$[-0.1, +2.2]$ and $[-0.2, +1.9]$ that include zero. The three estimators read
the same greedy evidence and order the experts alike: the Spearman correlation
between the estimated and the true training accuracies is $0.97$ for both,
although both overestimate every expert, OW-L by $0.14$ and OW-I by $0.10$ on
average. Unlike OW-L and OW-I, the jury can use answer distributions:
with them it reaches $76.6\%$, $5.3$ and $5.2$ points above
OW-L and OW-I ($[+3.8, +6.9]$ and $[+3.6, +6.7]$), and neither rule accepts a
distribution.

On greedy answers SkillAggregation reaches
$66.8\%$, above majority voting ($62.3\%$) and below the jury on the same input
($72.4\%$). On answer distributions it reaches $73.2\%$ without labels, above the
equal-weight pool ($70.8\%$), and the development set does not change this
($72.9\%$). The context-dependent variant, SkillAggregation-X, matches the base
model when $\lambda$ is chosen on the development set ($73.2\%$) and falls to
$49.6\%$ at the released $\lambda = 0.1$, which over-regularises its
context-dependent confusion matrices; at $\lambda \le 0.01$ it stays near $73\%$
from the tenth epoch onwards. The best observed result over the $40$ epochs is
$74.0\%$ for SkillAggregation-X ($\lambda = 0$, epoch $10$) and $73.5\%$ for the
base model ($\lambda = 0.1$, epoch $32$). Even this test-selected upper reference
remains $2.7$ points below the jury with repeated sampling ($3.1$ for the base
model), and below the supervised linear fusion ($75.9\%$).

\subsection{Replication across sub-pools}
\label[appendix]{app:pools}

\begin{table}[t]
\centering
\footnotesize
\setlength{\tabcolsep}{3pt}
\resizebox{\linewidth}{!}{\begin{tabular}{lrrrrrrrrrr}
\toprule
Jury & Pool & Pools & Jury & Degen. & \multicolumn{2}{c}{$-$ equal-weight pool} & \multicolumn{2}{c}{$-$ supervised fusion} & \multicolumn{2}{c}{$-$ best expert} \\
 & size & & & (\%) & $\Delta$ & $>0$ (\%) & $\Delta$ & $>0$ (\%) & $\Delta$ & $>0$ (\%) \\
\midrule
$\alpha = 8$ & 3 & 84 & 64.0 $\pm$ 13.7 & 0 & +2.1 $\pm$ 3.4 & 79 & -0.9 $\pm$ 2.9 & 42 & -0.7 $\pm$ 5.0 & 82 \\
 & 5 & 126 & 71.6 $\pm$ 6.9 & 0 & +4.9 $\pm$ 3.5 & 98 & +0.2 $\pm$ 1.1 & 63 & +1.3 $\pm$ 1.9 & 96 \\
 & 7 & 36 & 74.8 $\pm$ 3.2 & 0 & +4.7 $\pm$ 2.5 & 100 & +0.8 $\pm$ 0.8 & 86 & +1.8 $\pm$ 0.7 & 100 \\
 & 9 & 1 & 76.9 $\pm$ 0.0 & 0 & +5.4 $\pm$ 0.0 & 100 & +0.7 $\pm$ 0.0 & 100 & +1.9 $\pm$ 0.0 & 100 \\
\midrule
$\alpha = 1$ & 3 & 84 & 29.8 $\pm$ 24.3 & 69 & -32.2 $\pm$ 24.3 & 17 & -35.1 $\pm$ 24.9 & 8 & -35.0 $\pm$ 24.8 & 20 \\
 & 5 & 126 & 70.7 $\pm$ 8.4 & 1 & +4.0 $\pm$ 4.7 & 94 & -0.7 $\pm$ 3.9 & 50 & +0.4 $\pm$ 4.2 & 83 \\
 & 7 & 36 & 74.3 $\pm$ 3.1 & 0 & +4.2 $\pm$ 2.4 & 100 & +0.3 $\pm$ 0.3 & 78 & +1.3 $\pm$ 1.2 & 81 \\
 & 9 & 1 & 76.7 $\pm$ 0.0 & 0 & +5.2 $\pm$ 0.0 & 100 & +0.5 $\pm$ 0.0 & 100 & +1.7 $\pm$ 0.0 & 100 \\
\bottomrule
\end{tabular}
}
\caption{Every rule of the main comparison refitted on each sub-pool of $3$,
$5$ and $7$ of the nine experts, and on the full pool, over the same ten splits.
A pool's number is its mean over the splits, so the full-pool row averages
split-level differences and can differ by a tenth of a point from the
question-level paired differences of \Cref{tab:main}; each column gives the mean and
standard deviation over pools of that number, and $>0$ the share of pools on
which the jury is ahead; the three $-$ columns are the jury minus that rule. \emph{Degen.} is the share of pools on which the jury
falls ten or more points below majority voting, the operational criterion used
here to flag degenerate fits. The jury is fitted with the paper's evidence
weight, the multinomial $\alpha = 8$, and with the fractional $\alpha = 1$ of
the ablation. The pools overlap, so the spread is descriptive.}
\label{tab:pools}
\end{table}

\begin{figure}[t]
\centering
\includegraphics[width=\textwidth]{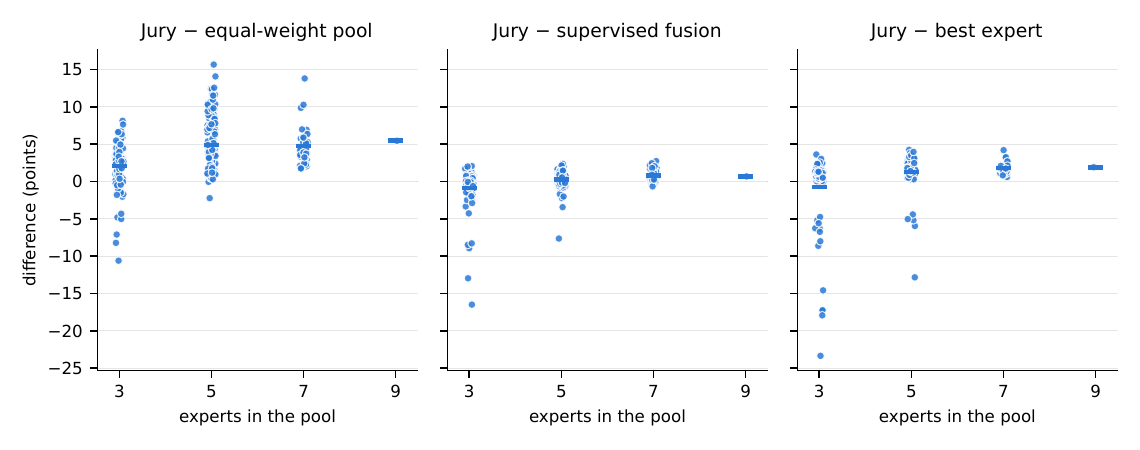}
\caption{The jury's margin over the equal-weight pool, the supervised fusion,
and the best expert across sub-pools of different sizes. Each dot is one
sub-pool; the horizontal bar marks the mean at each pool size.
\Cref{fig:pools-alpha} overlays the $\alpha = 1$ ablation.}
\label{fig:pools}
\end{figure}

The main experiments use a single nine-model ensemble. To test sensitivity to
pool composition, we repeat the comparison on every sub-pool of $3$, $5$ and $7$
experts drawn from it ($84 + 126 + 36 = 246$ pools), on the same ten splits,
with every rule fitted exactly as in \Cref{tab:main} on the pool's own experts: the jury by
EM on the training half, the supervised linear fusion on the training labels,
the equal-weight pool on nothing, and the best expert picked on the training
half. \Cref{tab:pools} and \Cref{fig:pools} report the distribution over pools
of the three paired differences the paper rests on.

\paragraph{Consistent across sub-pools.}
With the paper's evidence weight, $\alpha = 8$, the jury beats the equal-weight
pool on $79\%$ of the $84$ three-expert pools, $98\%$ of the $126$ five-expert
pools and all $36$ seven-expert pools, by $+2.1$, $+4.9$ and $+4.7$ points on
average. It is at or above the pool's best expert on $82\%$, $96\%$ and
$100\%$ of them, by $-0.7$, $+1.3$ and $+1.8$ points on average; on $14$ of
the $15$ three-expert and all $5$ five-expert exceptions the supervised fusion
does not beat that expert either: these are pools of weak experts on which
neither the jury nor the supervised fusion improves on the best expert. Against the
supervised fusion the jury is $-0.9$, $+0.2$ and $+0.8$ points on average and
within one point on $74\%$, $77\%$ and $67\%$ of the pools, with the
remainder split both ways; on the full pool it is $+0.7$.

\paragraph{Degenerate fits with fractional evidence weighting.}
The same experiment also shows why the model-implied multinomial weight is preferable in small pools.
With the fractional $\alpha = 1$ the jury still replicates from five experts
upwards (it beats the equal-weight pool on $94\%$ and $100\%$ of those pools
and is within a point of the supervised fusion on $85\%$ and $100\%$), but on
$58$ of the $84$ three-expert pools it collapses to an average accuracy of
$15.8\%$, far below chance. The collapse persists under the initialisations
we tried: EM converges to the same solution from constant starts of $0.6$ and
$0.9$ and from the plurality-agreement start. Its objective value exceeds that
obtained using the empirically estimated expert reliabilities. These checks
show that the failure is not specific to the default initialisation; they do
not establish local or global optimality. The multi-start analysis of
\Cref{app:multistart} shows that, on the full pool in the greedy $\alpha = 1$
regime, the default fit is not globally optimal: a degenerate fixed point with
higher likelihood exists on every split. No higher-likelihood alternative was
found for the sampled-distribution $\alpha = 8$ regime. At the three-expert
degenerate solution every $\pi_e$ is below $1/|\cand|$, every
weight of \Cref{prop:pool} is negative, and the verdict is the candidate with
the least mass weighted by $|w_e|$, $\arg\min_v \sum_e |w_e|\, p_e(v)$, which
is the candidate with the least total mass only when the weights share one
absolute value. The mechanism is
the fractional likelihood: with one unit of evidence per expert, the hypothesis
that every expert missed the truth costs $\sum_e \log\frac{1-\pi_e}{m}$, which
for small $\pi_e$ is about as cheap as the truth channel of a weak expert, so
once the pool is weak enough the model prefers to explain every configuration as
a shared miss. The multinomial weight $\alpha = 8$ scales every expert's
log-likelihood, the shared-miss channel included, so at fixed parameters it
changes no verdict; what it changes is the E-step, whose posteriors become far
more concentrated, and through them the fixed points EM reaches. From the
default initialisation we observed no degenerate fit with $\alpha = 8$ on any
of the $247$ pools, which is an empirical finding about these pools and not a
guarantee: random starts do reach a degenerate mode on $3$ of $200$ runs
(\Cref{app:multistart}). The collapse is also visible without labels: on one
split, the verdict of a degenerate fit agrees with the greedy plurality on
$17\%$ of the test questions on average, against $67\%$ for the other fits and
$66\%$ at $\alpha = 8$, and a threshold of $40\%$ flags $54$ of the $58$
degenerate fits and $2$ of the $26$ others.

\paragraph{What the choice costs.}
On the full pool the two weights are within noise ($76.9\%$ against $76.7\%$),
and with five and seven experts $\alpha = 8$ is at least as good as
$\alpha = 1$ on every column. The multinomial weight is therefore free where
the fractional one works and decisive where it failed on the pools tested. This confirms the
multinomial choice used by default; its one measurable cost is a more concentrated posterior
(\Cref{app:ablations}). A deployment on a pool of unknown strength can in any
case check the fit without labels by the share of questions on which the
verdict agrees with the plurality.

\subsection{Sensitivity to the EM start}
\label[appendix]{app:multistart}

\begin{table}[t]
\centering
\footnotesize
\setlength{\tabcolsep}{4pt}
\begin{tabular}{lrrrr}
\toprule
Start & Same solution & $\Delta$ log-lik.\ of the others & \multicolumn{2}{c}{Test accuracy (\%)} \\
 & as default & (nats) & min & max \\
\midrule
\multicolumn{5}{l}{\emph{Sampled distributions, $\alpha = 8$}} \\
Constant $0.6$ (default) & 10/10 & -- & 75.0 & 79.5 \\
Constant $0.3$ & 10/10 & -- & 75.0 & 79.5 \\
Constant $0.5$ & 10/10 & -- & 75.0 & 79.5 \\
Constant $0.9$ & 10/10 & -- & 75.0 & 79.5 \\
Plurality agreement & 10/10 & -- & 75.0 & 79.5 \\
Random, $20$ draws & 197/200 & $-9{,}568$ to $-9{,}889$ & 8.5 & 79.5 \\
\midrule
\multicolumn{5}{l}{\emph{Greedy answers, $\alpha = 1$}} \\
Constant $0.6$ (default) & 10/10 & -- & 71.3 & 74.5 \\
Constant $0.3$ & 0/10 & $+33$ to $+114$ & 7.6 & 9.2 \\
Constant $0.5$ & 10/10 & -- & 71.3 & 74.5 \\
Constant $0.9$ & 10/10 & -- & 71.3 & 74.5 \\
Plurality agreement & 10/10 & -- & 71.3 & 74.5 \\
Random, $20$ draws & 144/200 & $+33$ to $+114$ & 7.6 & 74.5 \\
\bottomrule
\end{tabular}

\caption{The EM of \Cref{sec:method} re-run on each of the ten splits of the
full pool from $25$ initialisations, in both evidence regimes of
\Cref{tab:main}. \emph{Same solution as default} counts the runs whose final
log-likelihood is within $10^{-3}$ nats of the default start's;
$\Delta$ log-lik.\ is the range, over the other runs, of their log-likelihood
minus the default start's; the accuracy columns are the extremes over the
runs and splits of that row.}
\label{tab:multistart}
\end{table}

For the sampled-distribution regime with $\alpha = 8$, we rerun EM on each of
the ten splits from $25$ initialisations: constants $0.3$, $0.5$, $0.6$ and
$0.9$, a plurality-agreement initialisation that sets $\pi_e$ to the mass
expert $e$ places on the plurality answer on average, and $20$ random vectors
sampled uniformly from $[0.05, 0.95]$ (\Cref{tab:multistart}). Of the $250$
runs, $247$ converge to the same solution as the default $0.6$ start to
numerical tolerance. The remaining three converge to a degenerate solution
whose log-likelihood is lower by approximately $9{,}600$ nats, where nats are log-likelihood units based on
the natural logarithm. Thus, on every split, the highest-likelihood solution found is the one reached by the default
start, with identical verdicts and reliabilities differing by at most
$2 \times 10^{-7}$ across the non-degenerate runs.

With greedy answers and $\alpha=1$, every split has a degenerate fixed point whose log-likelihood
is higher than that of the solution reached from the default $0.6$ initialisation.
The gap ranges from $33$ to $114$ nats across splits. At this mode, the
strongest experts are assigned very low reliabilities and consequently receive
negative weights, yielding approximately $9\%$ accuracy. The default $0.6$ start
never reaches this mode, whereas the constant $0.3$ start does on every split
and $19$ of $20$ random starts reach it on at least one split. Thus the greedy
fit reported in \Cref{tab:main} is not a global maximiser of the single-answer
likelihood: the default start converges to a lower-likelihood fixed point. In
contrast, in the sampled-distribution regime with
$\alpha = 8$, no higher-likelihood alternative was found in the $250$ runs.

% ===========================================================================
\section{Ablations}
\label[appendix]{app:ablations}
% ===========================================================================

\begin{figure}[t]
\centering
\includegraphics[width=\textwidth]{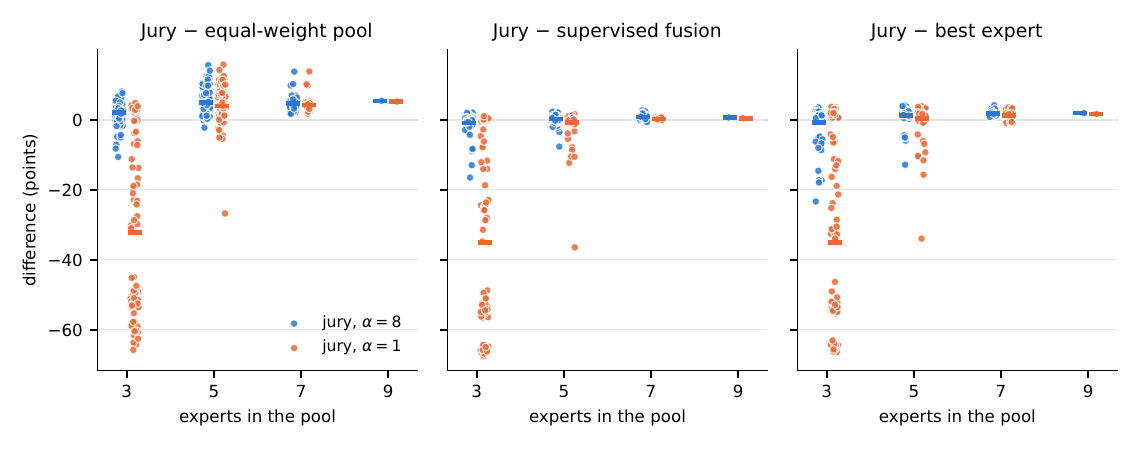}
\caption{\Cref{fig:pools} with the fractional $\alpha = 1$ overlaid. With three
experts, the lower cluster corresponds to the degenerate fits described in
\Cref{app:pools}; no such fits were observed with multinomial weighting from
the default initialisation on the tested pools.}
\label{fig:pools-alpha}
\end{figure}

\paragraph{Evidence weight.}
Sweeping $\alpha$ over $\{0.5, 1, 2, 4, 8\}$ gives $75.9$, $76.7$, $76.8$, $76.9$
and $76.9\%$. All values from $\alpha = 1$ upwards are within noise of the
default $\alpha = 8$, the exact multinomial weight for eight samples, and only
$\alpha = 0.5$ has a cost ($-1.0$ points). On the full pool the results are
therefore insensitive to this choice for $\alpha \ge 1$. On sub-pools of three
experts they are not: \Cref{app:pools} reports degenerate fits with
$\alpha = 1$ on most three-expert pools, but none with $\alpha = 8$ from the
default initialisation on the tested pools (\Cref{fig:pools-alpha}). What $\alpha$ does change on the full pool is the
calibration of the posterior, reported below.

\paragraph{Correlated errors.}
\Cref{tab:testimony} reports both halves of the check in \Cref{sec:ablations}.
On greedy answers a shared lure leaves the verdicts identical to those of the
independent model, except for a $-0.05$ change with the pool-wide grouping. With
repeated sampling the susceptibilities are fitted away from zero (the largest
averages $0.7$ by lineage and $0.8$ pool-wide over the ten splits) and the channel costs accuracy: $-0.6$ points by
lineage, on none of the ten splits positive, and $-0.4$ pool-wide. Under the
fractional weight $\alpha = 1$ every susceptibility collapsed to zero and the
verdicts were identical.
Measured against gold, the excess probability that two experts give the
\emph{same} wrong answer when both are wrong, over the rate expected from
independent errors, is $+10.3$ points for pairs sharing a pre-training lineage
and $+11.6$ points for pairs that do not (against $+1.2$ and $+3.5$ on the
original benchmark, since four-option questions concentrate the wrong answers).
Lineage therefore does not predict correlated error, which is a caution about
family-diversity heuristics. The observed dependence is therefore real, but
this parameterisation does not improve predictive accuracy.

\begin{table}[t]
\centering
\footnotesize
\begin{tabular}{lrrr}
\toprule
& Pairs & Agree on error & Excess over chance \\
\midrule
Same pre-training lineage & 10 & 42.0 & +10.3 \\
Different lineage & 26 & 43.9 & +11.6 \\
\midrule
\multicolumn{4}{l}{\emph{Accuracy with a shared-error channel switched on}} \\
\quad Greedy answers, no lure & -- & 72.3 & \\
\quad\quad + lineage lure & -- & 72.3 & $\pm$0.00~(identical) \\
\quad\quad + pool lure & -- & 72.3 & -0.05~(-0.65) \\
\quad With repeated sampling, no lure & -- & 76.9 & \\
\quad\quad + lineage lure & -- & 76.2 & -0.63~(-5.68) \\
\quad\quad + pool lure & -- & 76.4 & -0.42~(-1.59) \\
\bottomrule
\end{tabular}

\caption{Is there correlated error to model? Measured against gold for every pair
of experts, and then by offering the jury a shared-error channel over the same ten
splits. In the lower block, the parenthesised value is the paired $t$ statistic
of the difference over the ten splits.}
\label{tab:testimony}
\end{table}

\paragraph{Selective override.}
\Cref{tab:override} sweeps a threshold $\tau$ by which the verdict must beat the
plurality before it may overturn it. For the jury posterior and for the
discriminative scorer's margin alike, accuracy is highest at or near
$\tau = 0$: protecting the plurality stratum costs more than it returns. The
jury's posterior is so concentrated under the multinomial weight that
thresholds up to $0.5$ change almost nothing.

\begin{table}[t]
\centering
\footnotesize
\begin{tabular}{lrrrr}
\toprule
$\tau$ & Override rate & Accuracy & Plurality & Minority \\
\midrule
\multicolumn{5}{l}{\emph{Jury score, with repeated sampling}} \\
\quad $0$ (trust the jury) & 23.4 & 76.9 & 95.7 & 55.5 \\
\quad 0.1 & 23.3 & 76.9 & 95.8 & 55.5 \\
\quad 0.2 & 23.3 & 76.9 & 95.8 & 55.5 \\
\quad 0.3 & 23.2 & 76.9 & 95.8 & 55.4 \\
\quad 0.5 & 22.9 & 76.9 & 96.0 & 55.2 \\
\quad 0.7 & 22.2 & 76.8 & 96.1 & 54.6 \\
\midrule
\multicolumn{5}{l}{\emph{discriminative scorer, greedy answers}} \\
\quad 0.00 & 28.8 & 76.2 & 93.1 & 58.7 \\
\quad 0.12 & 25.5 & 76.1 & 94.7 & 55.4 \\
\quad 0.24 & 22.2 & 75.6 & 95.8 & 51.6 \\
\quad 0.36 & 18.5 & 74.6 & 97.1 & 46.3 \\
\quad $\infty$ (consensus) & 0.0 & 63.0 & 100.0 & 6.2 \\
\multicolumn{5}{l}{\emph{discriminative scorer, with repeated sampling}} \\
\quad 0.00 & 29.3 & 76.7 & 92.9 & 60.3 \\
\quad 0.12 & 25.8 & 76.6 & 94.6 & 57.0 \\
\quad 0.24 & 22.6 & 76.1 & 96.0 & 53.0 \\
\quad 0.36 & 19.5 & 75.4 & 97.1 & 48.8 \\
\quad $\infty$ (consensus) & 0.0 & 63.0 & 100.0 & 6.2 \\
\bottomrule
\end{tabular}

\caption{Selective override: the verdict must beat the plurality by $\tau$ to
stand. $\tau = \infty$ returns the plurality answer always; $\tau = 0$ always
takes the verdict. Accuracy never rises as $\tau$ grows.}
\label{tab:override}
\end{table}

\paragraph{Calibration.}
Expected calibration error is $0.223$ over ten bins: mean top-1 confidence is
$0.991$ against $76.9\%$ accuracy, with a Brier score of $0.448$ and a negative
log-likelihood of $4.97$ where the gold answer is on the table. The
multinomial weight makes the posterior nearly one-hot, since it counts eight
units of evidence per expert; with the fractional weight $\alpha = 1$ the same
numbers are $0.155$, $0.917$, $0.388$ and $0.984$. The posterior ranks
candidates well but needs recalibration before being read as a probability,
and more so at $\alpha = 8$ than at $\alpha = 1$.

% ===========================================================================
\section{The support control}
\label[appendix]{app:support}
% ===========================================================================

\begin{table}[t]
\centering
\small
\resizebox{\linewidth}{!}{\begin{tabular}{llrrrrr}
\toprule
& & \multicolumn{2}{c}{All $1560$ questions} & & \multicolumn{2}{c}{Multiple choice only} \\
\cmidrule(lr){3-4} \cmidrule(lr){6-7}
Rule & Labels & Observed & Four options & $\Delta$~[95\% CI] & Observed & Four options \\
\midrule
Optimal Weight, greedy answers & yes & 73.3 $\pm$ 1.1 & 73.0 $\pm$ 1.4 & -0.2~[-1.2, +0.9] & 73.5 & 73.2 \\
Jury (ours), greedy answers & no & 72.3 $\pm$ 1.1 & 72.0 $\pm$ 1.3 & -0.4~[-1.2, +0.4] & 72.4 & 72.0 \\
Supervised linear fusion & yes & 76.2 $\pm$ 1.4 & 76.2 $\pm$ 1.4 & +0.0~[+0.0, +0.0] & 75.9 & 75.9 \\
Jury (ours) + repeated sampling & no & 76.9 $\pm$ 1.3 & 76.9 $\pm$ 1.3 & +0.0~[+0.0, +0.0] & 76.6 & 76.6 \\
\bottomrule
\end{tabular}
}
\caption{The candidate set read off the pool's answers (\emph{Observed}, the
paper's convention) against the task's four options on every multiple-choice
question (\emph{Four options}), with every parameter re-estimated under each.
Same answers, same ten splits. $\Delta$ is the paired difference, four options
minus observed, with a $95\%$ bootstrap interval over questions; GSM8K keeps the
observed support in both columns, so the right-hand block restricts to the
$1440$ multiple-choice questions.}
\label{tab:support}
\end{table}

\Cref{sec:model} builds $\cand(q)$ from the answers the pool produced and
defines the emission model on that support. On multiple choice the answer
space is known before any expert is asked, so the support need not be
reconstructed from the testimony. This control fixes $\cand(q) = \{A, B, C,
D\}$ on every multiple-choice question: $m = 3$ always, an option nobody
proposed is a candidate with zero empirical mass, and a question on which the
pool is unanimous is no longer a singleton and enters EM, since under a fixed
space unanimity is evidence about the reliabilities. About ten of the roughly
$720$ multiple-choice training questions per split are unanimous on greedy
answers. Every parameter is re-estimated under the fixed support; nothing is
changed at prediction time only. GSM8K keeps the observed support, its answer
space being unbounded. The supervised linear fusion and Optimal Weight are
run under the same convention, so the control applies on both sides of the
supervision line; for Optimal Weight the four-option column is the
task-defined $K = 4$ of \Cref{tab:main} and the observed column the refit of
\Cref{app:supervised}.

\Cref{tab:support} gives the result. With repeated sampling nothing changes:
the sampled support already contains all four options on all but four of the
$1440$ questions (\Cref{sec:setup}), so the fitted reliabilities agree to
three decimals and every verdict is identical, for the jury and for the
supervised fusion alike. On greedy answers, where the pool proposes $2.98$
options on average, the four-option support moves the jury by $-0.4$ points
($[-1.2, +0.4]$) and Optimal Weight by $-0.2$ ($[-1.2, +0.9]$), neither
significant; the greedy reliabilities rise by $0.01$ to $0.04$, since the
unanimous questions now count as agreement. On this benchmark the
construction of the support is therefore not what the results rest on. The
control says nothing about the other assumptions of the model, conditional
independence in particular, which \Cref{app:ablations} examines separately.

% ===========================================================================
\section{Proof of \texorpdfstring{\Cref{prop:identifiability}}{Proposition \ref{prop:identifiability}}}
\label[appendix]{app:proofident}
% ===========================================================================

\begin{proof}
Write $\cand=\cand(q)$, with $|\cand|=K\ge3$.
Throughout the proof, all expectations are taken within the
group of questions with $|\cand(q)|=K$.
As the reliabilities are question-independent, identifying them
within this group identifies the same vector $\pi$ for all questions.

For a question $q$, let expert $e$ provide a fixed number
$n_e\ge1$ of sampled answers $a_{e,1},\ldots,a_{e,n_e}$.
Its empirical answer distribution is
\[
  p_e(v)=\frac{1}{n_e}\sum_{r=1}^{n_e}
    \mathbf 1\{a_{e,r}=v\},
  \qquad v\in\cand,
\]
where we drop $q$ in the notation.
As each sample follows the emission model \eqref{eq:emission},
we get
\begin{align*}
  \mathbb E[p_e(v)\mid z,q]
  &=\frac{1}{n_e}\sum_{r=1}^{n_e}
    P(a_{e,r}=v\mid z,q)\\
  &=P(a_e=v\mid z,q).
\end{align*}
Thus the empirical frequencies are unbiased for the emission
probabilities.

Define the centred frequencies and rescaled reliabilities by
\begin{align*}
  x_e(v)&=p_e(v)-\frac{1}{K},\\
  s_e&=\frac{K\pi_e-1}{K-1}.
\end{align*}
Equivalently, $\pi_e=((K-1)s_e+1)/K$.
In particular, $s_e=0$ corresponds exactly to chance-level
reliability, $\pi_e=1/K$.

Using the conditional mean of $p_e(v)$ above, the emission model
gives
\begin{align*}
  &\mathbb E[x_e(v)\mid z,q]\\
  &\quad=
  \begin{cases}
    \pi_e-\dfrac{1}{K}, & v=z,\\[4pt]
    \dfrac{1-\pi_e}{K-1}-\dfrac{1}{K}, & v\ne z.
  \end{cases}
\end{align*}
Substituting the expression for $\pi_e$ simplifies to
\begin{equation}
  \mathbb E[x_e(v)\mid z,q]
  =
  \begin{cases}
    \dfrac{K-1}{K}s_e, & v=z,\\[4pt]
    -\dfrac{1}{K}s_e, & v\ne z.
  \end{cases}
  \label{eq:ident-centred-mean}
\end{equation}

For distinct experts $i,j,k$, define the second- and third-order
centred agreement moments
\begin{align}
  A^2_{ij}
  &:=\mathbb E\!\left[
    \sum_{v\in\cand}x_i(v)x_j(v)
  \right],
  \label{eq:ident-pair-definition}\\
  A^3_{ijk}
  &:=\mathbb E\!\left[
    \sum_{v\in\cand}x_i(v)x_j(v)x_k(v)
  \right].
  \label{eq:ident-triple-definition}
\end{align}
These moments are determined by the joint distribution of the
observed empirical distributions, thus their definitions do not
involve the latent true answer.

By assumption, the empirical distributions are conditionally
independent across experts given $(z,q)$.
Their centred responses therefore satisfy
\begin{align*}
  &\mathbb E[x_i(v)x_j(v)\mid z,q]\\
  &\quad=\mathbb E[x_i(v)\mid z,q]\,
    \mathbb E[x_j(v)\mid z,q],
\end{align*}
and the analogous factorisation holds for three distinct experts.
Each support contains one candidate with $v=z$ and $K-1$
candidates with $v\ne z$.
Using \eqref{eq:ident-centred-mean}, their contributions to the
second-order moment give
\begin{align}
  A^2_{ij}
  &=s_i s_j\left[
    \frac{(K-1)^2}{K^2}+\frac{K-1}{K^2}
  \right] \notag\\
  &=\frac{K-1}{K}s_i s_j.
  \label{eq:ident-pair}
\end{align}
The first term inside the brackets comes from $v=z$;
the second sums the contributions of the $K-1$ alternatives.
For the third-order moment, the product of three negative
coefficients at an alternative candidate changes the sign
of the second term:
\begin{align}
  A^3_{ijk}
  &=s_i s_j s_k\left[
    \frac{(K-1)^3}{K^3}-\frac{K-1}{K^3}
  \right] \notag\\
  &=\frac{(K-1)(K-2)}{K^2}s_i s_j s_k.
  \label{eq:ident-triple}
\end{align}
The conditional sums just computed do not depend on $z$ or $q$
once the support size is fixed.
Averaging over the latent answer and the questions in this group
therefore leaves both expressions unchanged, giving the
displayed moments.

Whenever $A^2_{ij}\ne0$, dividing \eqref{eq:ident-triple}
by \eqref{eq:ident-pair} cancels $s_i s_j$ and yields
\begin{align*}
  \frac{A^3_{ijk}}{A^2_{ij}}
  &=\frac{K-2}{K}s_k,\\
  s_k&=\frac{K}{K-2}\frac{A^3_{ijk}}{A^2_{ij}}.
\end{align*}
Since $K\ge3$, this coefficient is nonzero.
Returning to the original parameterisation gives
\begin{align}
  \pi_k
  &=\frac{(K-1)s_k+1}{K} \notag\\
  &=\frac{1}{K}
    +\frac{K-1}{K-2}\frac{A^3_{ijk}}{A^2_{ij}}.
  \label{eq:ident-reconstruction}
\end{align}

It remains to check that a usable pair $i,j$ exists for every
target expert $k\in\experts$.
By \eqref{eq:ident-pair},
\[
  A^2_{ij}\ne0
  \quad\Longleftrightarrow\quad
  s_i\ne0\ \text{and}\ s_j\ne0.
\]
The assumption supplies at least three experts with $s_e\ne0$.
For any $k$, at least two of them are different from $k$;
choosing these as $i,j$ makes the denominator in
\eqref{eq:ident-reconstruction} nonzero.

The formula also covers chance-level experts: if $s_k=0$,
then $A^3_{ijk}=0$ and it returns $\pi_k=1/K$.

The moments \eqref{eq:ident-pair-definition} and
\eqref{eq:ident-triple-definition} are functions of the joint
distribution of the unlabelled expert outputs on the
task-defined supports.
In particular, a usable pair can be recognised from
$A^2_{ij}\ne0$ without knowing the experts' reliabilities
in advance.

Now suppose two reliability vectors $\pi$ and $\pi'$ induce the
same observable distribution.
They then have the same second- and third-order moments.
For each target expert $k$, choose a pair with $A^2_{ij}\ne0$
as above.
Applying \eqref{eq:ident-reconstruction} to the shared moments
gives $\pi_k=\pi'_k$.
This holds for every $k\in\experts$, so $\pi=\pi'$,
proving identifiability.
\end{proof}

\end{document}